\documentclass[runningheads]{llncs}

\usepackage[utf8]{inputenc}

\usepackage{times}
\usepackage{soul}
\usepackage{url}
\usepackage{graphicx}
\usepackage{amsmath}
\usepackage{amssymb}
\usepackage{booktabs}
\usepackage{algorithm}
\usepackage{algorithmicx,algpseudocode}
\usepackage{xcolor}
\usepackage{xspace}
\usepackage{multirow}
\usepackage{hyperref}

\usepackage{latexsym}

\newenvironment{proof-sketch}{{\noindent{\emph{Proof (Sketch).}}}}

\newcommand{\Const}{\mathsf{Const}}
\newcommand{\perfectRef}{\mathsf{PerfectRef}}
\newcommand{\mapRef}{\mathsf{MapRef}}
\newcommand{\rewr}{\mathsf{rew}}

\newcommand{\per}{\mbox{\bf .}}
\newcommand{\body}{\mathit{set}}
\newcommand{\query}{\mathit{query}}

\newcommand{\LOGSPACE}{\textsc{LogSpace}\xspace}

\newcommand{\NEXPTIME}{\textsc{NExpTime}\xspace}
\newcommand{\coNEXPTIME}{\textsc{coNExpTime}\xspace}

\newcommand{\NP}{\textrm{NP}\xspace}
\newcommand{\DP}{\textrm{DP}\xspace}
\newcommand{\coNP}{\textrm{coNP}\xspace}

\newcommand{\Tetwop}{{\Theta_2^p}\xspace}

\newcommand{\myi}{(\emph{i})\xspace}
\newcommand{\myii}{(\emph{ii})\xspace}
\newcommand{\myiii}{(\emph{iii})\xspace}
\newcommand{\myiv}{(\emph{iv})\xspace}
\newcommand{\myv}{(\emph{v})\xspace}

\newcommand{\C}{\mathcal{C}}

\newcommand{\I}{\mathcal{I}}

\renewcommand{\L}{\mathcal{L}}
\newcommand{\M}{\mathcal{M}}

\renewcommand{\O}{\mathcal{O}}

\newcommand{\Q}{\mathcal{Q}}
\newcommand{\R}{\mathcal{R}}
\renewcommand{\S}{\mathcal{S}}

\newcommand{\V}{\mathcal{V}}

\newcommand{\ISA}{\sqsubseteq}

\newcommand{\dllite}{\textit{DL-Lite}\xspace}

\newcommand{\dlliter}{\textit{DL-Lite}_{\R}\xspace}

\newcommand{\owltwoql}{\texttt{OWL\,2\,QL}\xspace}
\newcommand{\owltwo}{\texttt{OWL\,2}\xspace}

\newcommand{\tup}[1]{\langle#1\rangle}

\newcommand{\cert}{\mathit{cert}}

\newcommand{\wrt}{w.r.t.\xspace}
\newcommand{\dom}{\mathit{dom}}

\begin{document}

\title{QDEF and Its Approximations in OBDM}

\titlerunning{QDEF and Its Approximations in OBDM}

\author{Gianluca Cima\inst{1}\orcidID{0000-0003-1783-5605} \and
Federico Croce\inst{2}\orcidID{0000-0001-6779-4624} \and
Maurizio Lenzerini\inst{2}\orcidID{0000-0003-2875-6187}}
\authorrunning{G. Cima et al.}

\institute{CNRS \& University of Bordeaux, Bordeaux, France\\
\email{gianluca.cima@u-bordeaux.fr}\\
\and
Sapienza University of Rome, Rome, Italy\\
\email{\{croce, lenzerini\}@diag.uniroma1.it}}

\sloppy

\maketitle

\begin{abstract}
    Given an input dataset (i.e., a set of tuples), query definability in Ontology-based Data Management (OBDM) amounts to find a query over the ontology whose certain answers coincide with the tuples in the given dataset. We refer to such a query as a \emph{characterization} of the dataset with respect to the OBDM system. Our first contribution is to propose approximations of perfect characterizations in terms of recall (complete characterizations) and precision (sound characterizations). A second contribution is to present a thorough complexity analysis of three computational problems, namely verification (check whether a given query is a perfect, or an approximated characterization of a given dataset), existence (check whether a perfect, or a best approximated characterization of a given dataset exists), and computation (compute a perfect, or best approximated characterization of a given dataset).    
\end{abstract}


\section{Introduction}
As first introduced for relational databases~\cite{Zloof75,TTCYS14,BCS16,BR17}, query definability is the reverse engineering task that, given a set of tuples and a database, aims at finding a query whose answers over such database are exactly the tuples in the set. In other words, the goal of this task is to derive an intensional definition (the query) of an extensionally defined set. Over the years, researchers have found several interesting applications of this problem, spanning from simplifying query formulation by non-experts, to debugging facilities for data engineers. Moreover, the query definability has been studied as a useful tool for data exploration, data analysis, usability, data security and more~\cite{MM2019,MLVP2017}. With the rise of Machine Learning (ML), we argue that this topic could be also beneficial for providing meaningful reformulations of what is called a training dataset in any typical supervised ML-based classification task. In this context, the training set used in a classification task is seen as a set of tuples in a database schema, and the query derived by solving the query definability problem results into an intensional definition of the input training set. In a sense, the expression derived can be used as an explanation of the intensional properties of the training set. The idea is that an intensional characterization of the training set can help understanding the behaviour of a classifier, a very important task for wide and safe adoption of machine learning and data mining technologies, especially in dealing with bias.

In this paper, we address the problem of query definability in the context of Ontology-based Data Management (OBDM), which is a paradigm for accessing data using a conceptual representation of the domain of interest expressed as an ontology. OBDM relies on a three-level architecture, consisting of the schema of the data layer $\S$ (which we assume constituted by a relational database), the ontology $\O$, a declarative and explicit representation of the domain of interest for the organization, and the mapping $\M$ between the two. Consequently, an OBDM specification is formalized as the triple $J=\tup{\O,\S,\M}$ which, together with an $\S$-database $D$, form a so-called OBDM system $\Sigma = \tup{J,D}$. In this context, we are going to tackle the problem of query definability by leveraging the notion of evaluation of a query with respect to an OBDM system, in turn based on the notion of \textit{certain answers} to a query over an OBDM system.
Intuitively, given an OBDM system $\Sigma=\tup{J,D}$ and a $D$-dataset $\lambda$, our goal is to derive a query expression over $\O$ that suitably \emph{characterizes}
$\lambda$ \wrt\ $\Sigma$. In other words, we aim at deriving a ``good'' definition of $\lambda$ using a query expressed over the concepts and roles of the ontology of $J$.

Inspired by the works in~\cite{BJS18,Ortiz19} about query definability in Description Logics (DLs), we consider the query whose certain answers with respect to $\Sigma=\tup{J,D}$ is exactly $\lambda$ as the perfect characterization for $\lambda$. We note that, since in this paper we tackle the query definability problem in OBDM, differently from the works in~\cite{BJS18,Ortiz19}, this work has the added complexity of considering the mapping layer of the OBDM system, which is, to the best of our knowledge, novel to this field.
This work has also been inspired by the \textit{concept learning} tools presented in~\cite{BLWB18,FRAE2018,SM15}, and by the notion of \textit{query abstraction} \cite{CiLP19,Cima20,CCCLP21}. We differ from the former because in that work the goal is to learn a concept expression capturing a given dataset, whereas our goal is to derive a full-blown query that, evaluated over the ontology, returns the dataset as answers. We differ from the latter because, although the goal is stil to derive a query expression over the ontology, in that work the input is a query over the data layer, whereas in query definability the input is a set of tuples. It follows that the two tasks are completely different and require different technical solutions: in the present work we aim at finding a query over the ontology such that the certain answers of the query \wrt the OBDM system are equal to the given specific dataset, whereas in \cite{CiLP19,Cima20,CCCLP21}, the goal is to find a query over the ontology such that the certain answers of the query are equal to the evaluation of the given query over the database schema, for all possible databases of the OBDM system. In the framework section of this paper we will better characterize the relationship between the two notions of query definability and query abstraction.

Virtually all the above-mentioned works point out that in
many cases a perfect ontological characterization of a given dataset does not exist. We argue that, in these cases, reasonable and useful ontological characterizations can still be provided. In particular, we propose to resort to suitable approximations of the perfect
characterizations, in terms of recall and precision. To this end, we
introduce the notions of sound and complete characterizations. The former is
a query whose certain answers form a subset of the $D$-dataset $\lambda$ in
input, whereas the certain answers of the latter, form a superset of
the $D$-dataset $\lambda$. Obviously, we are interested in computing the best
approximated characterizations, which we call maximally sound and minimally
complete characterizations, respectively. A maximally sound (resp.,
minimally complete) characterization is a sound (resp., complete) characterization such that no other sound (resp., complete)
characterization exists that better approximates the $D$-dataset $\lambda$.

This paper provides the following contributions. 
\begin{itemize}
\item We present a general, formal framework for the various notions of
  ontological characterizations mentioned above. 
  The framework includes the definition of three tasks that are relevant for reasoning about characterizations of a dataset,
  namely verification (verify whether a given query is a sound,
  complete, or perfect characterizations), computation (compute a
  characterization of a certain type), and existence (check whether a
  characterization of a certain type exists).
\item We provide computational complexity results for the three
  reasoning tasks mentioned above in a scenario that uses the most common languages in the OBDM literature, namely where the ontology language is $\dlliter$, the mapping language is GLAV, and the query language to express characterizations is the one of union of conjunctive queries. As for the two decision problems of
  verification and existence, we provide both upper bounds and
  matching lower bounds. As for the computation task, we provide
  algorithms for computing perfect, minimally complete, and maximally sound characterizations, provided they exist.
\end{itemize}

The paper is organized as follows. After the preliminaries in Section~\ref{sec:Preliminaries}, Section~\ref{sec:Framework} illustrates the framework, and Sections~\ref{sec:Verification}, \ref{sec:Computation}, and~\ref{sec:Existence} present the
results on the three reasoning tasks, i.e., verification, computation
and existence, respectively. Finally, Section~\ref{sec:Conclusion} concludes the paper by
discussing possible future work.


\section{Preliminaries}\label{sec:Preliminaries}

We recall some notations and languages about relational databases~\cite{AbHV95}, Description Logics (DLs)~\cite{BCMNP03}, and the Ontology-based Data Management (OBDM) paradigm~\cite{Lenz11}.

\paragraph{Databases, Datasets, and Queries:}

A \emph{relational database schema} (or simply \emph{schema}) $\S$ is a finite set of predicate symbols, each with a specific arity. Given a schema $\S$, an \emph{$\S$-database} $D$ is a finite set of \emph{facts} satisfying all integrity constraints in $\S$ whose form is $s(\vec{c})$, where $s$ is an $n$-ary predicate symbol of $\S$, and $\vec{c}=(c_1,\ldots,c_n)$ is an $n$-tuple of constants, each taken from a countable infinite set of symbols denoted by $\Const$. We denote by $\dom(D)$ the finite set of constants occurring in $D$. Observe that $\dom(D)\subseteq \Const$.

Given a schema $\S$ and an $\S$-database $D$, a \emph{$D$-dataset} $\lambda$ of arity $n$ is a finite set of $n$-tuples $\vec{c}$ of constants occurring in $D$, i.e., $\lambda \subseteq \dom(D)^n$.

A \emph{query} $q_{\S}$ over a schema $\S$ is an expression in a certain query language $\Q$ using the predicate symbols of $\S$ and arguments of predicates are \emph{variables}, i.e., we disallow constants to occur in queries. Each query has an associated arity. The \emph{evaluation} of a query $q_{\S}$ of arity $n$ over an $\S$-databases $D$ is a set of \emph{answers} $q_{\S}^D$, each answer being an $n$-tuple of constants occurring in $\dom(D)$, i.e., $q_{\S}^D \subseteq \dom(D)^n$. We are particularly interested in \emph{conjunctive queries} and unions thereof.

A \emph{conjunctive query (CQ)} over a schema $\S$ is an expression of the form $q_{\S}= \{\vec{x} \mid \exists \vec{y} \per \phi(\vec{x},\vec{y})\}$ such that \myi $\vec{x}=(x_1,\ldots,x_n)$, called the \emph{target list} of $q_{\S}$, is an $n$-tuple of \emph{distinguished variables}, where $n$ is the arity $q_{\S}$ \myii $\vec{y}=(y_1,\ldots,y_m)$ is an $m$-tuple of \emph{existential variables}; and \myiii $\phi(\vec{x},\vec{y})$, called the \emph{body} of $q_{\S}$, is a finite conjunction of atoms of the form $s(v_1,\ldots,v_p)$, where $s$ is a $p$-ary predicate symbol of $\S$ and $v_i$ is either a distinguished or an existential variable, i.e., $v_i \in \vec{x} \cup \vec{y}$, for each $i=[1,p]$. 
Variables belong to a countable infinite set of symbols denoted by $\V$, where $\Const \cap \V = \emptyset$. A \emph{union of conjunctive queries (UCQ)} is a finite set of CQs with same arity, called its \emph{disjuncts}.

For a conjunction of atoms $\phi(\vec{x},\vec{y})$, we denote by $\body(\phi)$ the set of all the atoms occurring in $\phi$. For a set of atoms $\C$ and a tuple $\vec{c}=(c_1,\ldots,c_n)$ of constants, 
we denote by $\query(\C,\vec{c})$ the CQ $\{\vec{x} \mid \exists \vec{y} \per \phi(\vec{x},\vec{y})\}$, where \myi $\phi(\vec{x},\vec{y})$ is the conjunction of all the atoms occurring in the set of atoms $\C'$, where $\C'$ is obtained from $\C$ by replacing everywhere each constant $c_i$ occurring in $\vec{c}$ with a fresh variable $x_{c_i}$ and each constant $c$ not occurring in $\vec{c}$ with a fresh variable $y_c$, \myii $\vec{x}=(x_{c_1},\ldots,x_{c_n})$, and \myiii $\vec{y}$ is the tuple of all variables occurring in $\C'$ that do not occur in $\vec{x}$. 

Following the terminology of~\cite{CaDa15}, we say that a query $q_{\S}$ over a schema $\S$ \emph{defines a $D$-dataset $\lambda$ inside an $\S$-database $D$} if $q_{\S}^{D}=\lambda$, and say that $\lambda$ is \emph{$\Q$-definable inside $D$}, for a query language $\Q$, if there exists a query $q_{\S} \in \Q$ that defines $\lambda$ inside $D$.

Given a set of atoms $\C$, we denote by $\dom(\C)$ the set of all constants and variables occurring in a set of atoms $\C$. Observe that $\dom(\C) \subseteq \Const \cup \V$. Let $\C_1$ and $\C_2$ be two sets of atoms. We say that a function $h: \dom(\C_1) \rightarrow \dom(\C_2)$ is a \emph{homomorphism} from $\C_1$ to $\C_2$ if $h(\C_1) \subseteq h(\C_2)$, where $h(\C_1)$ is the image of $\C_1$ under $h$, i.e., $h(\C_1)=\{h(\alpha) \mid \alpha \in \C_1\}$ with $h(s(t_1,\ldots,t_n)) = s(h(t_1),\ldots,h(t_n))$ for each atom $\alpha = s(t_1,\ldots,t_n)$. For two sets of atoms $\C_1$ and $\C_2$ and two tuples of terms $\vec{t_1}$ and $\vec{t_2}$, we write $(\C_1, \vec{t_1}) \rightarrow (\C_2,\vec{t_2})$ if there is a 
function $h$ from $\dom(\C_1) \cup \vec{t_1}$ to $\dom(\C_2) \cup \vec{t_2}$ such that \myi $h$ is a homomorphism from $\C_1$ to $\C_2$, and \myii $h(\vec{t_1})=\vec{t_2}$ (where, for a tuple of terms $\vec{t}=(t_1,\ldots,t_n)$, $h(\vec{t})=(h(t_1),\ldots,h(t_n))$), $(\C_1, \vec{t_1}) \not\rightarrow (\C_2,\vec{t_2})$ otherwise.

Observe that for an $\S$-database $D$ and a CQ $q_{\S}=\{\vec{x} \mid \exists \vec{y} \per \phi(\vec{x},\vec{y})\}$ over $\S$ of arity $n$, the set of answers $q_{\S}^{D}$ corresponds to the set of $n$-tuples $\vec{c}$ of constants occurring in $D$ for which $(\body(\phi),\vec{x}) \rightarrow (D,\vec{c})$.

\paragraph{Syntax and Semantics of $\dlliter$:}
DLs are fragments of First-order logic languages using only unary and binary predicates, called \emph{atomic concepts} and \emph{atomic roles}, respectively. In this paper, a \emph{DL ontology} (or simply \emph{ontology}) $\O$ is a TBox (``Terminological Box'') expressed in a specific DL, that is, a set of assertions stating general properties of concepts and roles built according to the syntax of the specific DL, which represents the intensional knowledge of a modeled domain.

We are interested in DL ontologies expressed in $\dlliter$, the member of the $\dllite$ family~\cite{CDLLR07} that underpins $\owltwoql$, i.e., the $\owltwo$ profile especially designed for efficient query answering~\cite{W3Crec-OWL2-Profiles}. A $\dlliter$ ontology $\O$ is a finite set of \emph{assertions} of the form:
$$
    \begin{array}[t]{lll}
        B_1 \ISA B_2 & \quad R_1 \ISA R_2 & \quad \mbox{(concept/role inclusion)}\\[1mm]
        B_1 \ISA \neg B_2 &\quad R_1 \ISA \neg R_2 &\quad \mbox{(concept/role disjointness)}
    \end{array}
$$
where $B_1$, $B_2$ are basic concepts, i.e., expressions of the form $A$, $\exists P$, or $\exists P^-$, with $A$ and $P$ an atomic concept (atomic concepts include the universal one $\top$ and the bottom one $\bot$) and an atomic role, respectively, and $R_1$ and $R_2$ basic roles, i.e., expressions of the form $P$, or $P^-$.

Given a $\dlliter$ ontology $\O$, we denote by $V_{\O}$ the \emph{$\O$-violation query}, i.e., the boolean UCQ obtained by including a disjunct of the form $\{() \mid \exists y \per A_1(y) \wedge A_2(y)\}$ (respectively, $\{() \mid \exists y_1,y_2 \per A_1(y_1) \wedge R(y_1,y_2)\}$, $\{() \mid \exists y_1, y_2, y_3 \per R_1(y_1,y_2) \wedge R_2(y_1,y_3)\}$, and $\{() \mid \exists y_1, y_2 \per R_1(y_1,y_2) \wedge R_2(y_1,y_2)\}$) for each disjointness assertion $A_1 \ISA \neg A_2$ (respectively, $A_1 \ISA \neg \exists R$ or $\exists R \ISA \neg A_1$, $\exists R_1 \ISA \neg \exists R_2$, and $R_1 \ISA \neg R_2$) occurring in $\O$, where an atom of the form $R(y,y')$ stands for either $P(y,y')$ if $R$ denotes an atomic role $P$, or $P(y',y)$ if $R$ denotes the inverse of an atomic role, i.e., $R=P^-$.

The semantics of DL ontologies is specified through the notion of interpretation: an \emph{interpretation} $\I$ for an ontology $\O$ is a pair $\I = \tup{\Delta^{\I},\cdot^{\I}}$, where the \emph{interpretation domain} $\Delta^{\I}$ 
is a non-empty, possibly infinite set of constants, and the \emph{interpretation function} $\cdot^{\I}$ assigns to each atomic concept $A$ a set of domain objects $A^{\I}\subseteq\Delta^{\I}$, and to each atomic role $P$ a set of pairs of domain objects $P^{\I}\subseteq\Delta^{\I}\times\Delta^{\I}$. For the constructs of $\dlliter$, the interpretation function extends to other basic concepts and basic roles as follows: $\top^{\I}=\Delta^{\I}$, $\bot^{\I}=\emptyset$, $(\exists P)^{\I}=\{o\mid\exists o'.\ (o,o')\in P^{\I}\}$, and $(P^-)^{\I}=\{(o,o')\mid (o',o)\in P^{\I}\}$.
We often treat interpretations $\I$ for ontologies $\O$ as a (possibly infinite) set of facts over (the atomic concepts and roles in the alphabet of) $\O$.

We say that an interpretation $\I$ for an ontology $\O$ satisfies $\O$, denoted by $\I \models \O$, if $\I$ satisfies every assertion in $\O$. For the $\dlliter$ assertions, an interpretation $\I$ satisfies a concept inclusion assertion $B_1 \ISA B_2$ (respectively, role inclusion assertion $R_1 \ISA R_2$) if $B_1^{\I} \subseteq B_2^{\I}$ (respectively, $R_1^{\I} \subseteq R_2^{\I}$), and it satisfies a concept disjointness assertion $B_1 \ISA \neg B_2$ (respectively, role disjointness assertion $R_1 \ISA \neg R_2$) if $B_1^{\I} \cap B_2^{\I} = \emptyset$ (respectively, $R_1^{\I} \cap R_2^{\I} = \emptyset$).

Whenever we speak about queries $q_{\O}$ over ontologies $\O$, we mean queries in a certain language $\Q$ using the atomic concepts and roles in the alphabet of $\O$ as predicates. For a UCQ $q_{\O}$ over a $\dlliter$ ontology $\O$, we denote by $\perfectRef(\O,q_{\O})$ the UCQ computed by executing the algorithm $\perfectRef$~\cite{CDLLR07} on $\O$ and $q_{\O}$. 

\paragraph{Ontology-based Data Management:} According to~\cite{PLCD*08,Lenz11}, an \emph{Ontology-based Data Management (OBDM)} specification is a triple $J=\tup{\O,\S,\M}$, where $\O$ is a DL ontology, $\S$ is a relational database schema, also called \emph{source schema}, and $\M$ is a \emph{mapping}, i.e., a finite set of assertions over the signature $\S \cup \O$ relating the source schema $\S$ to the ontology $\O$. An OBDM system is a pair $\Sigma=\tup{J,D}$, where $J=\tup{\O,\S,\M}$ is an OBDM specification and $D$ is an $\S$-database.

The semantics of an OBDM system $\Sigma=\tup{\tup{\O,\S,\M},D}$ is given in terms of interpretations $\I=\tup{\Delta^{\I},\cdot^{\I}}$ for $\O$ in which the interpretation function $\cdot^{\I}$ further assigns to each constant $c \in \dom(D)$ a domain object $c \in \Delta^{\I}$. Specifically, we say that an interpretation $\I$ for $\O$ is a \emph{model} of an OBDM system $\Sigma=\tup{\tup{\O,\S,\M},D}$ if \myi $\I \models \O$, and \myii the pair $\tup{\I,D} \models \M$. We say that an OBDM system $\Sigma$ is \emph{consistent} if it has at least one model, \emph{inconsistent} otherwise.

The set of \emph{certain answers} of a query $q_{\O}$ over an ontology $\O$ \wrt\ an OBDM system $\Sigma=\tup{J,D}$ with $J=\tup{\O,\S,\M}$, denoted by $\cert_{q_{\O},J}^{D}$, is the set of tuples of constants $(c_1, \ldots, c_n)$ such that $(c_1^{\I}, \ldots, c_n^{\I}) \in q_{\O}^{\I}$ for each model $\I$ of $\Sigma$, where $\I$ is seen as a set of facts over $\O$. If $\Sigma$ is inconsistent, then the set of certain answers of any query $q_{\O}$ over $\O$ \wrt $\Sigma$ is simply the set of all possible tuples of constants occurring in $D$ whose arity is the one of the query. We say that two queries $q_1$ and $q_2$ are equivalent \wrt an OBDM system $\Sigma=\tup{J,D}$ if $\cert_{q_1,J}^D=\cert_{q_2,J}^D$.

As for the mapping component of an OBDM system, in this paper we are interested in \emph{GLAV} assertions~\cite{DoHI12}, which are assertions of the form
$
    q_{\S} \rightarrow q_{\O},
$
where $q_{\S}$ and $q_{\O}$ are CQs over $\S$ and over $\O$, respectively, with the same target list $\vec{x}=(x_1,\ldots,x_n)$. Special cases of GLAV assertions highly considered in the data integration literature are GAV and LAV assertions~\cite{Lenz02}: in a GAV (resp., LAV) mapping, $q_{\O}$ (resp., $q_{\S}$) is simply an atom without existential variables. A GLAV (resp., GAV, LAV, GAV$\cap$LAV) mapping is a finite set of GLAV (resp., GAV, LAV, both GAV and LAV) assertions.

Given a GLAV mapping $\M$ relating $\S$ to $\O$, an interpretation $\I$ for $\O$, and an $\S$-database $D$, we have that $\tup{\I,D} \models \M$ if $(c_1,\ldots,c_n) \in q_{\S}^D$ implies $(c_1^{\I},\ldots,c_n^{\I}) \in q_{\O}^{\I}$ for each mapping assertion $q_{\S} \rightarrow q_{\O}$ occurring in $\M$ and for each possible tuple $(c_1,\ldots,c_n)$ of constants occurring in $D$. 

Let $J=\tup{\O,\S,\M}$ be an OBDM specification where $\O=\emptyset$, i.e., $\O$ has no assertions, and $\M$ is a GLAV mapping. From results of~\cite{FrLM99,CDLV12}, it is well-known that, given a UCQ $q_{\O}$ over $\O$, by splitting the GLAV mapping $\M$ into a GAV mapping followed by a LAV mapping over an intermediate alphabet, it is always possible to compute a UCQ over $\S$, denoted by $\mapRef(\M,q_{\O})$, such that $\mapRef(\M,q_{\O})^D = \cert_{q_{\O},J}^D$ for each $\S$-database $D$. 

Let now $J=\tup{\O,\S,\M}$ be an OBDM specification where $\O$ is a $\dlliter$ ontology and $\M$ is a GLAV mapping. For a UCQ $q_{\O}$ over $\O$, we denote by $\rewr_{q_{\O},J}$ the following UCQ over $\S$: $\rewr_{q_{\O},J}:=\mapRef(\M,\perfectRef(\O,q_{\O}))$. By combining the above observation with results of~\cite{CDLLR07}, we have that \myi $\cert_{q_{\O},J}^D=\rewr_{q_{\O},J}^D$, for each UCQ $q_{\O}$ over $\O$ and for each $\S$-database $D$ such that $\tup{J,D}$ is consistent, and \myii $\tup{J,D}$ is inconsistent if and only if $\rewr_{V_\O,J}^D=\{\tup{}\}$, for each $\S$-database $D$. We note that $\dlliter$ is insensitive to the adoption of the unique name assumption for UCQ answering~\cite{ACKZ09}.

\paragraph{Canonical Structure:} Given an $\S$-database $D$ and a GLAV mapping $\M$ relating a schema $\S$ to an ontology $\O$, the \emph{chase}~\cite{CaGK13} of $D$ with respect to $\M$, denoted by $\M(D)$, is the set of atoms computed as follows: \myi we start with $\M(D):=\emptyset$; then \myii for every GLAV assertion $\{\vec{x} \mid \exists \vec{y} \per \phi_{\S}(\vec{x},\vec{y})\} \rightarrow \{\vec{x} \mid \exists \vec{z} \per \varphi_{\O}(\vec{x},\vec{z})\}$ in $\M$ and for every tuple of constants $\vec{c}$ such that $(\body(\phi_{\S}),\vec{x}) \rightarrow (D,\vec{c})$, we add to $\M(D)$ the image of the set of atoms $\body(\varphi_{\O})$ under $h'$, that is, $\M(D):=\M(D) \cup h'(\varphi_{\O}(\vec{x},\vec{z}))$, where $h'$ extends $h$ by assigning to each variable $z$ occurring in $\vec{z}$ a different fresh variable of $\V$ still not present in $\dom(\M(D))$. Observe that $\M(D)$ is guaranteed to be finite and can be always computed in exponential time.  

We conclude this section with the following observation used in the technical development of the next sections. 
Let $\Sigma=\tup{\tup{\O,\S,\M},D}$ be an OBDM system where $\O$ is a $\dlliter$ ontology and $\M$ is a GLAV mapping. We call the \emph{canonical structure} of $\O$ with respect to $\M$ and $D$, denoted by $\C_{\O}^{\M(D)}$, the (possibly infinite) set of atoms obtained by first computing $\M(D)$ as described before, and then by chasing $\M(D)$ with respect to the inclusion assertions of $\O$ as described in~\cite[Definition~5]{CDLLR07} but using the alphabet $\V$ of variables whenever a new element is needed in the chase. Observe that this latter is a \emph{fair} deterministic strategy, i.e., it is such that if at some point an assertion is applicable, then it will be eventually applied. By combining results of~\cite[Proposition~4.2]{FKMP05} with~\cite[Theorem~29]{CDLLR07}, it is well-known that, for a UCQ $q_{\O}=\{\vec{x_1} \mid \exists \vec{y_1} \per \phi^1_{\O}(\vec{x_1},\vec{y_1})\} \cup \ldots \cup \{\vec{x_p} \mid \exists \vec{y_p} \per \phi^p_{\O}(\vec{x_p},\vec{y_p})\}$ over $\O$ and a tuple of constants $\vec{c}$, if $\Sigma=\tup{J,D}$ is consistent, then we have $\vec{c} \in \cert_{q_{\O},J}^D$ if and only if $(\body(\phi^i_{\O}),\vec{x_i}) \rightarrow (\C_{\O}^{\M(D)},\vec{c})$ for some $i \in [1,p]$.


\section{Framework}\label{sec:Framework}

In what follows, $\Sigma=\tup{J,D}$ refers to an OBDM system where $J=\tup{\O,\S,\M}$ is an OBDM specification and $D$ is an $\S$-database. Intuitively, given a set $\lambda$ of $n$-tuples of constants occurring in $D$ (i.e., $\lambda$ is a $D$-dataset of arity $n$), we aim at finding a query $q_{\O}$ over $\O$ in a certain query language $\Q$ characterizing $\lambda$ \wrt\ the OBDM system $\Sigma$. Since the evaluation of queries is based on certain answers, we are naturally led to the following definition.

\begin{definition}\label{def:explanation}
    $q_\O \in \Q$ is a \emph{perfect $\Sigma$-characterization} of $\lambda$ in the query language $\Q$, if $cert^D_{q_\O, J} = \lambda$.
\end{definition}

Clearly, if a perfect $\Sigma$-characterization of $\lambda$ exists, then it is unique up to $\Sigma$-equivalence, and therefore in the following we will always refer to \emph{the} perfect $\Sigma$-characterization of $\lambda$ in the query language $\Q$. 

\begin{example}\label{ex:mainexample}
     Let $\Sigma=\langle J,D\rangle$ be as follows. $J=\langle\O,\S,\M\rangle$ is the OBDM specification such that $\O=\{\textsf{MathStudent} \sqsubseteq \textsf{ Student}, \textsf{ ForeignStudent} \sqsubseteq \textsf{ Student}\}$, $\S=\{s_1,s_2,s_3,s_4,s_5\}$, and $\M$ contains the GAV assertions:
    \begin{align*}
        & \{(x) \mid s_1(x)\} \rightarrow \{(x) \mid \textsf{ Student}(x)\}\\
        & \{(x) \mid s_2(x)\} \rightarrow \{(x) \mid \textsf{ Student}(x)\}\\
        & \{(x_1,x_2) \mid s_3(x_1,x_2)\}\rightarrow \{(x_1,x_2) \mid \textsf{ EnrolledIn}(x_1,x_2)\}\\
        & \{(x) \mid \exists y \per s_3(x,y) \wedge s_4(y)\}\rightarrow \{(x) \mid \textsf{ MathStudent}(x)\}\\
        & \{(x) \mid \exists y \per s_3(x,y) \wedge s_5(y)\}\rightarrow \{(x) \mid \textsf{ ForeignStudent}(x)\}
        \vspace*{-.12cm}
    \end{align*}
    And the $\S$-database is $D=\{s_1(c_4), s_2(c_3), s_4(b_1), s_5(d_1), \\ s_3(c_1,b_1), s_3(c_2,d_1), s_3(c_3,e_1), s_3(c_4,e_2), s_3(c_5,e_3)\}$. For the $D$-dataset $\lambda=\{(c_1),(c_2),(c_3)\}$, since $q_{\O}^1=\{(x) \mid \textsf{ Student}(x)\}$ and $q_{\O}^2=\{(x) \mid \exists y \per \textsf{EnrolledIn}(x,y)\}$ are such that $\cert_{q_{\O}^1,J}^D=\{(c_1),(c_2),(c_3),(c_4)\}$ and $\cert_{q_{\O}^2,J}^D=\{(c_1),(c_2),(c_3),(c_4),(c_5)\}$, and since $q_{\O}^3=\{(x) \mid \textsf{MathStudent}(x)\}$ and $q_{\O}^4=\{(x) \mid \textsf{ ForeignStudent}(x)\}$ are such that $\cert_{q_{\O}^3,J}^D=\{(c_1)\}$ and $\cert_{q_{\O}^4,J}^D=\{(c_2)\}$, one can verify that no perfect $\Sigma$-characterization of $\lambda$ in UCQ exists.
\end{example}

Notice the difference with the notion of \emph{abstraction}~\cite{Cima20,CCCLP21}, introduced in~\cite{Cima17} and studied under various scenarios~\cite{LuMS18,CiLP19,CiLP20}. In abstraction, we are given an OBDM specification $J=\tup{\O,\S,\M}$ and a query $q_{\S}$ over $\S$, and the aim is to find a query $q_{\O}$ over $\O$, called \emph{the perfect $J$-abstraction of $q_{\S}$}, such that $\cert_{q_{\O},J}^D=q_{\S}^D$ \emph{for each $\S$-database $D$} for which $\tup{J,D}$ is consistent. Conversely, here we are also given an $\S$-database $D$, and instead of a query $q_{\S}$ we have a set of tuples $\lambda$ of constants taken from $D$, and the aim is to find a query $q_{\O}$ over $\O$ such that $\cert_{q_{\O},J}^D=\lambda$. 
The following proposition establishes the relationship between the notion of characterization introduced here and the notion of abstraction.
\begin{proposition}
    Let $\Sigma=\tup{J,D}$ be a consistent OBDM system, $\lambda$ be a $D$-dataset, and $q_{\S}$ be a query that defines $\lambda$ inside $D$. If a query $q_{\O} \in \Q$ is the perfect $J$-abstraction of $q_{\S}$, then $q_{\O}$ is the perfect $\Sigma$-characterization of $\lambda$ in $\Q$.
\end{proposition}

\begin{proof}
    Suppose $q_{\O} \in \Q$ is the perfect $J$-abstraction of $q_{\S}$, i.e., $\cert_{q_{\O},J}^{D'}=q_{\S}^{D'}$ for each $\S$-database $D'$ for which $\tup{J,D'}$ is consistent. Since $\Sigma=\tup{J,D}$ is consistent and since $q_{\S}^D=\lambda$ by assumption that $q_{\S}$ defines $\lambda$ inside $D$, we have that $\cert_{q_{\O},J}^D=q_{\S}^D=\lambda$, which means that $q_{\O}$ is a perfect $\Sigma$-characterization of $\lambda$ in $\Q$.
\end{proof}

The next example shows that the converse of the above proposition does not necessarily hold, thus stressing the fact that the two problems are indeed different.

\begin{example}\label{ex:ControEsemp}
    Let $\Sigma=\tup{J,D}$ be as follows: \myi $J=\tup{\O,\S,\M}$ is such that $\O=\emptyset$, $\S=\{s_1,s_2\}$, and $\M=\{m_1,m_2\}$ with $m_1=\{(x) \mid s_1(x)\} \rightarrow \{(x) \mid A(x)\}$ and $m_2=\{(x) \mid s_2(x)\} \rightarrow \{(x) \mid A(x)\}$; and \myii $D=\{s_1(c)\}$.
    
    For the $D$-dataset $\lambda=\{(c)\}$, one can verify that $q_{\S}=\{(x) \mid s_1(x)\}$ is such that $q_{\S}^D=\lambda$ and that $q_{\O}=\{(x) \mid A(x)\}$ is such that $\cert_{q_{\O},J}^D=\lambda$, i.e., $q_{\O}$ is the perfect $\Sigma$-characterization of $\lambda$ in CQ. However, the query $q_{\O}$ is not a perfect $J$-abstraction of $q_{\S}$, since for the $\S$-database $D'=\{s_2(c)\}$ we have $\cert_{q_{\O},J}^{D'}=\{(c)\}$ whereas $q_{\S}^{D'}=\emptyset$.
\end{example}

Clearly, the more expressive the query language $\Q$, the more likely we can express the implicit relationship between the tuples in $\lambda$ by means of the operators in $\Q$, and therefore the more likely the perfect characterization in $\Q$ exists. Unfortunately, the next example shows that, even without any restriction on the query language, perfect characterizations are not guaranteed to exist even in trivial cases.

\begin{example}\label{ex:ControEsemp2}
    Recall the OBDM specification $J$ of the previous example, and let $\Sigma=\tup{J,D}$ be the OBDM system with $D=\{s_1(c_1),s_2(c_2)\}$. For the $D$-dataset $\lambda=\{(c_1)\}$, one can trivially verify that, whatever is the query language $\Q$, there is no query $q_{\O} \in \Q$ for which $\cert_{q_{\O},J}^D=\lambda$.
\end{example}

Note the importance of the role played by the mapping $\M$ in order to reach this conclusion. Indeed, if we replace $m_2$ with $\{(x) \mid s_2(x)\} \rightarrow \{(x) \mid B(x)\}$, then the perfect $\Sigma$-characterization of $\lambda$ would be the CQ $\{(x) \mid A(x)\}$.

Borrowing the ideas from~\cite{CiLP19} to remedy situations where perfect abstractions do not exist, we now introduce approximations of perfect characterizations in terms of recall (complete) and precision (sound).

\begin{definition}
    $q_\O \in \Q$ is a \emph{complete} (resp., \emph{sound}) $\Sigma$-\emph{characterization} of $\lambda$ in the query language $\Q$, if $\lambda \subseteq cert^D_{q_\O, J}$ (resp., $cert^D_{q_\O, J} \subseteq \lambda$).
\end{definition}

\begin{example}
    Refer to Example~\ref{ex:mainexample}. We have that $q_{\O}^1$ and $q_{\O}^2$ are complete $\Sigma$-characterization of $\lambda$, whereas $q_{\O}^3$ and $q_{\O}^4$ are sound $\Sigma$-characterization of $\lambda$.
\end{example}

As the above example manifests, there may be several complete and sound characterizations relative to a query language $\Q$. In those cases, the interest is unquestionably in those that best approximate the perfect one.

\begin{definition}
    $q_\O$ is a $\Q$-\emph{minimally complete} (resp., $\Q$-\emph{maximally sound}) $\Sigma$-\emph{characterization} of $\lambda$, if $q_{\O}$ is a complete (resp., sound) $\Sigma$-characterization of $\lambda$ in $\Q$ and there is no $q'_\O \in \Q$ such that \myi $q'_\O$ is a \emph{complete} (resp., sound) $\Sigma$-characterization of $\lambda$ and \myii $cert^D_{q'_\O, J} \subset cert^D_{q_\O, J}$ (resp., $cert^D_{q_\O, J} \subset cert^D_{q'_\O, J}$).
\end{definition}

\begin{example}
    Refer again to Example~\ref{ex:mainexample}. The CQ $q_{\O}^1$ is a UCQ-minimally complete $\Sigma$-characterization of $\lambda$, whereas $q_{\O}^2$ is not. Both $q_{\O}^3$ and $q_{\O}^4$ are CQ-maximally sound $\Sigma$-characterizations of $\lambda$, but neither of them is a UCQ-maximally sound $\Sigma$-characterization of $\lambda$. Indeed, a UCQ-maximally sound $\Sigma$-characterization of $\lambda$ is $q_{\O}^5=q_{\O}^3 \cup q_{\O}^4$.
\end{example}


Given this general framework, there are (at least) three computational problems to consider, with respect to an ontology language $\L_{\O}$, a mapping language $\L_{\M}$, and a query language $\Q$. Given an OBDM system $\Sigma=\tup{\tup{\O,\S,\M},D}$ and a $D$-dataset $\lambda$, where $\O \in \L_{\O}$ and $\M \in \L_{\M}$:
\begin{itemize}
    \item \emph{Verification:} given $q_{\O} \in \Q$, check whether $q_{\O}$ is a perfect (resp., complete, sound) $\Sigma$-characterization of $\lambda$.
    \item \emph{Computation:} compute the perfect in $\Q$ (resp., $\Q$-minimally complete, or $\Q$-maximally sound) $\Sigma$-characterization of $\lambda$, provided it exists.
    \item \emph{Existence:} check whether there exists a perfect in $\Q$ (resp., $\Q$-minimally complete, or $\Q$-maximally sound) $\Sigma$-characterization of $\lambda$.
\end{itemize}

In what follows, if not otherwise stated, we refer to the following scenario which considers by far the most popular languages for the OBDM paradigm: \myi $\L_{\O}$ is $\dlliter$, \myii $\L_{\M}$ is GLAV, and \myiii $\Q$ is UCQ.

In this scenario, there are two interesting properties that are worth mentioning. First, since the UCQ language allows for the conjunction (resp., union) operator, if an UCQ-minimally complete (resp., UCQ-maximally sound) $\Sigma$-characterization of $\lambda$ exists, then it is unique up to $\Sigma$-equivalence. 

\begin{proposition}\label{prop:Uniq}
    If $q_1$ and $q_2$ are UCQ-minimally complete (resp., UCQ-maximally sound) $\Sigma$-characterizations of $\lambda$, then they are equivalent \wrt\ $\Sigma$.
\end{proposition}

\begin{proof}
    Similar to~\cite[Proposition~3.2]{Cima20}.
\end{proof}

Due to the above property, in what follows we simply refer to \emph{the} UCQ-minimally complete (resp., UCQ-maximally sound) $\Sigma$-characterization of $\lambda$.

Second, as expected, in this scenario perfect characterizations are less likely to exist than in the plain relational database case.


\begin{proposition}
    Let $\Sigma=\tup{J,D}$ be a consistent OBDM system, and $\lambda$ be a $D$-dataset. If there exists a perfect $\Sigma$-characterization of $\lambda$ in UCQ, then $\lambda$ is UCQ-definable inside $D$.
\end{proposition}

\begin{proof}
    Suppose there exists a perfect $\Sigma$-characterization of $\lambda$ in UCQ, i.e., there is a UCQ $q_{\O}$ over $\O$ for which $\cert_{q_{\O},J}^D=\lambda$. Recall from Section~\ref{sec:Preliminaries} that the UCQ $\rewr_{q_{\O},J}$ over $\S$ is such that $\cert_{q_{\O},J}^{D'}=\rewr_{q_{\O},J}^{D'}$ for each $\S$-database $D'$ for which $\tup{J,D'}$ is consistent. Since $\Sigma=\tup{J,D}$ is consistent, we have $\lambda=\cert_{q_{\O},J}^{D}=\rewr_{q_{\O},J}^{D}$, from which immediately follows that $\rewr_{q_{\O},J}$ defines $\lambda$ inside $D$, and thus $\lambda$ is UCQ-definable inside $D$.
\end{proof}

In general, the converse of the above proposition does not hold. Indeed, in Example~\ref{ex:ControEsemp2}, while there is no perfect $\Sigma$-characterization of $\lambda$ in any query language $\Q$, the CQ $q_{\S}=\{(x) \mid s_1(x)\}$ witnesses that $\lambda$ is CQ-definable inside $D$.



\section{Verification}\label{sec:Verification}

We now define the verification problems for $X$-query definability ($X$-VQDEF), where $X$=\{Perfect, Complete, Sound\}. These decision problems are parametric with respect to the ontology language $\L_{\O}$ to express $\O$, the mapping language $\L_{\M}$ to express $\M$, and the query language $\Q$ to express $q_{\O}$. 
\noindent
\begin{table}[h]
    \centering
    \framebox{
            \begin{tabular}{p{0.17\linewidth}p{0.68\linewidth}}
                \textsc{Problem}: & \textbf{X-VQDEF($\L_{\O}$, $\L_{\M}$, $\Q$)} \\
                \textsc{Input}: & An OBDM system $\Sigma=\tup{\tup{\O,\S,\M},D}$, a $D$-dataset $\lambda$, and a query $q_{\O} \in \Q$ over $\O$, where $\O \in \L_{\O}$ and $\M \in \L_{\M}$.\\
                \textsc{Question}: & Is $q_{\O}$ a \textbf{X} $\Sigma$-characterization of $\lambda$?
            \end{tabular}
    }
\end{table}

In what follows, given a syntactic object $x$ such as a query, an ontology, or a mapping, we denote by $\sigma(x)$ its size.

\begin{theorem}\label{th:UBC}
    Complete-VQDEF($\dlliter$, GLAV, UCQ) is in $\NP$.
\end{theorem}

\begin{proof}
    We now show how to check whether $q_{\O}$ is a complete $\Sigma$-characterization of $\lambda$ (i.e., $\lambda \subseteq \cert_{q_{\O},J}^D$) in $\NP$, where $\Sigma=\tup{J,D}$ with $J=\tup{\O,\S,\M}$.

    Let $n$ be the arity of the tuples in the $D$-dataset $\lambda$. For each $n$-tuple of constants $\vec{c} \in \lambda$, we first guess \myi a CQ $q'_{\O}$ over $\O$ which is either of arity $n$ and size at most $\sigma(q_{\O})$, or a boolean one capturing a disjointness assertion $d$ (e.g., $\{() \mid \exists y \per A_1(y) \wedge A_2(y)\}$ capturing $d=A_1 \ISA \neg A_2$); \myii a sequence $\rho_{\O}$ of ontology assertions; \myiii a CQ $q_{\S}$ over $\S$ of size at most $\sigma(\M) \cdot \sigma(q'_{\O})$ which is either of arity $n$ and of the form $\{\vec{x} \mid \exists \vec{y} \per \phi_{\S}(\vec{x},\vec{y})\}$, or a boolean one of the form $\{() \mid \exists \vec{y} \per \phi_{\S}(\vec{y})\}$; \myiv a sequence $\rho_{\M}$ of mapping assertions; and \myv a function $f$ from the variables occurring in $q_{\S}$ to $\dom(D)$. 
    
    Then, we check in polynomial time whether \myi by means of $\rho_{\O}$, either we can rewrite a disjunct of $q_{\O}$ into $q'_{\O}$ through $\O$ (i.e., $q'_{\O} \in \perfectRef(\O,q_{\O})$), or we can rewrite a disjunct of $V_{\O}$ into $q'_{\O}$ through $\O$ (i.e., $q'_{\O} \in \perfectRef(\O,V_{\O})$); \myii by means of $\rho_{\M}$ we can rewrite $q'_{\O}$ into $q_{\S}$ through $\M$ (i.e., $q_{\S} \in \mapRef(\M,q'_{\O})$, and thus either $q'_{\O} \in \rewr_{q_{\O},J}$ or $q'_{\O} \in \rewr_{V_{\O},J}$); and finally \myiii $f$ consists in a homomorphism witnessing either $(\body(\phi_{\S}),\vec{x}) \rightarrow (D,\vec{c})$, i.e., $\vec{c} \in q_{\S}^D$ (and therefore $\vec{c} \in \rewr_{q_{\O},J}^D$, which means $\vec{c} \in \cert_{q_{\O},J}^D$), or $(\body(\phi_{\S}),()) \rightarrow (D,())$, i.e., $D \models q_{\S}$ (and therefore $\rewr_{V_{\O},J}^D=\{\tup{}\}$, which means that $\Sigma$ is inconsistent and thus $\vec{c} \in \cert_{q_{\O},J}^D$ by definition).
\end{proof}

\begin{theorem}\label{th:UBS}
    Sound-VQDEF($\dlliter$, GLAV, UCQ) is in $\coNP$.
\end{theorem}

\begin{proof}
    We now show how to check whether $q_{\O}$ is not a sound $\Sigma$-characterization of $\lambda$ (i.e., $\cert_{q_{\O},J}^D \not\subseteq \lambda$) in $\NP$, where $\Sigma=\tup{J,D}$ with $J=\tup{\O,\S,\M}$.

    We first guess \myi a tuple of constants $\vec{c}$, and, exactly as in the proof of Theorem~\ref{th:UBC}, \myii $q'_{\O}$, $\rho_{\O}$, $q_{\S}$, $\rho_{\M}$, and $f$. Then, we check in polynomial time whether \myi $\vec{c}$ contains only constants from $\dom(D)$ and $\vec{c} \not\in \lambda$ (i.e., $\vec{c} \in \dom(D)^n \setminus \lambda$), and \myii using $q'_{\O}$, $\rho_{\O}$, $q_{\S}$, $\rho_{\M}$, and $f$, we follow exactly the same polynomial time procedure in the proof of Theorem~\ref{th:UBC} to check whether $\vec{c} \in \cert_{q_{\O},J}^D$.
\end{proof}

Recall that a decision problem is in $\DP$ if and only if it is the conjunction of a decision problem in $\NP$ and a decision problem in $\coNP$~\cite{PaYa84}. 
Since $q_{\O}$ is a perfect $\Sigma$-characterization of $\lambda$ if and only if it is both a sound, and a complete $\Sigma$-characterization of $\lambda$, we immediately derive the following upper bound.

\begin{corollary}\label{ch:UBP}
    Perfect-VQDEF($\dlliter$, GLAV, UCQ) is in $\DP$.
\end{corollary}

We now provide matching lower bounds. We show that they already hold for the same, very simple, fixed OBDM system $\Sigma$ and dataset $\lambda$, and for single, unary CQs as queries.

\begin{theorem}\label{th:vLB}
    There is an OBDM system $\Sigma=\tup{\tup{\O,\S,\M},D}$ such that $\O=\emptyset$ and $\M$ is a GAV$\cap$LAV mapping, and a $D$-dataset $\lambda$ containing only a unary tuple for which the problem Complete-VQDEF($\emptyset$, GAV$\cap$LAV, CQ) (resp., Sound-VQDEF($\emptyset$, GAV$\cap$LAV, CQ), Perfect-VQDEF($\emptyset$, GAV$\cap$LAV, CQ)) is $\NP$-hard (resp., $\coNP$-hard, $\DP$-hard).
\end{theorem}

\begin{proof}
 Let $\Sigma=\tup{J,D}$ be the OBDM system such that \myi $J=\tup{\O,\S,\M}$ is an OBDM specification in which $\O=\emptyset$ is an empty ontology whose alphabet contains two atomic roles $P_1$ and $P_2$, $\S=\{s_1,s_2\}$, and $\M$ contains the following two GAV$\cap$LAV assertions: 
    \begin{align*}
        & \{(x_1,x_2) \mid s_1(x_1,x_2)\} \rightarrow \{(x_1,x_2) \mid P_1(x_1,x_2)\},\\
        &\{(x_1,x_2) \mid s_2(x_1,x_2)\} \rightarrow \{(x_1,x_2) \mid P_2(x_1,x_2)\},
    \end{align*}
    which simply mirrors source predicate $s_i$ to atomic role $P_i$, for $i=[1,2]$, and \myii $D$ is the $\S$-database composed of the following facts: 
    \begin{align*}
        & \{s_1(x,y) \mid x=\{r',g',b'\} \text{ and } y=\{r',g',b'\} \text{ and } x \neq y\}\cup\\
        & \{s_1(x,y) \mid x=\{r,g,b,y\} \text{ and } y=\{r,g,b,y\} \text{ and } x \neq y\}\cup\\
        & \{s_2(x,c_3) \mid x=\{r',g',b'\}\} \cup \{s_2(x,c_4) \mid x=\{r,g,b,y\}\}.
    \end{align*}

    \noindent Let, moreover, $\lambda$ be the $D$-dataset $\lambda=\{(c_4)\}$.

    Let $G=(V,E)$ be a finite and undirected graph without loops or isolated nodes, where $V=\{y_1,\ldots,y_n\}$. We define a CQ $q_{G}=\{(x) \mid \exists \vec{y} \per \phi_{\O}(x,\vec{y})\}$ over $\O$ as follows: 
    $$
        q_{G}=\{(x) \mid \exists y_1,\ldots,y_n \per \bigwedge_{(y_i,y_j) \in E} (P_1(y_i,y_j))~\wedge \bigwedge_{y_i \in V} (P_2(y_i,x))\}
    $$

    Finally, notice that the CQ $q_{G}$ can be constructed in $\LOGSPACE$ from an input graph $G$. 
    
    By inspecting the OBDM system $\Sigma=\tup{J,D}$, for any graph $G$, the set of certain answers $\cert_{q_{G},J}^D$ is an element of the power set of $\{(c_3),(c_4)\}$. More specifically, the following property holds:

    \begin{claim}
        For both $i=3$ and $i=4$, we have that a graph $G=(V,E)$ is $i$-colourable if and only if $(c_i) \in \cert_{q_{G},D}^D$.
    \end{claim}

    \begin{proof}
        First of all, notice that $\C_{\O}^{\M(D)}$ is composed of the following facts:
        \begin{align*}
            & \{P_1(x,y) \mid x=\{r',g',b'\} \text{ and } y=\{r',g',b'\} \text{ and } x \neq y\}\cup\\
            & \{P_1(x,y) \mid x=\{r,g,b,y\} \text{ and } y=\{r,g,b,y\} \text{ and } x \neq y\}\cup\\
            & \{P_2(x,c_3) \mid x=\{r',g',b'\}\} \cup \{P_2(x,c_4) \mid x=\{r,g,b,y\}\}.
        \end{align*}

        ``\textbf{Only-if part:}'' Suppose $G=(V,E)$ is 3-colourable (resp., 4-colourable), that is, there exists a function $f:V\rightarrow \{r',g',b'\}$ (resp., $f:V\rightarrow \{r,g,b,y\}$) such that $f(y_i)\neq f(y_j)$ for each $(y_i,y_j)\in E$. Let $\phi_{\O}$ be the body of $q_{G}$, and consider the extension of $f$ which assigns to the distinguished variable $x$ of $q_{G}$ the constant $c_3$ (resp., $c_4$). It can be readily seen that $f$ consists in a homomorphism from $\body(\phi_{\O})$ to $\C_{\O}^{\M(D)}$ such that $f(x)=c_3$ (resp., $f(x)=c_4$). In other words, $f$ witnesses that $(\body(\phi_{\O}),(x)) \rightarrow (\C_{\O}^{\M(D)},(c_3))$ (resp., $(\body(\phi_{\O}),(x)) \rightarrow (\C_{\O}^{\M(D)},(c_4))$). Thus, $(c_3) \in \cert_{q_{G},J}^D$ (resp., $(c_4) \in \cert_{q_{G},J}^D$), as required.

        ``\textbf{If part:}'' Suppose $G=(V,E)$ is not 3-colourable (resp., not 4-colourable), that is, each possible function $f:V\rightarrow \{r',g',b'\}$ (resp., $f:V\rightarrow \{r,g,b,y\}$) is such that $f(y_i)=f(y_j)$ for some $(y_i,y_j)\in E$. Clearly, this implies that $(\body(\phi_{\O}),(x)) \not\rightarrow (\C_{\O}^{\M(D)},(c_3))$ (resp., $(\body(\phi_{\O}),(x)) \not\rightarrow (\C_{\O}^{\M(D)},(c_4))$). Thus, $(c_3) \not\in \cert_{q_{G},J}^D$ (resp., $(c_4) \not\in \cert_{q_{G},J}^D$), as required. \qed
    \end{proof}

    With the above property at hand, and the fact that $\cert_{q_{G},J}^D$ is an element of the power set of $\{(c_3),(c_4)\}$ for each possible finite and undirected graph $G=(V,E)$ without loops or isolated nodes, we are now ready to prove the claimed lower bounds.
    
    As for the complete case, the proof of $\NP$-hardness is by a $\LOGSPACE$ reduction from the \emph{4-colourability problem}, which is $\NP$-complete~\cite{GaJS76}. In particular, a graph $G$ is 4-colourable if and only if $(c_4) \in \cert_{q_{G},J}^D$, i.e., if and only if $\lambda \subseteq \cert_{q_{G},J}^D$. So, checking whether the CQ $q_{G}$ is a complete $\tup{\O,\Sigma}$-explanation of $\lambda$ is $\NP$-hard.

    As for the sound case, the proof of $\coNP$-hardness is by a $\LOGSPACE$ reduction from the \emph{complement of 3-colourability problem}, which is $\coNP$-complete~\cite{GaJS76}. In particular, a graph $G$ is not 3-colourable if and only if $(c_3) \not\in \cert_{q_{G},J}^D$, i.e., if and only if $\cert_{q_{G},J}^D \subseteq \lambda$. So, checking whether the CQ $q_{G}$ is a sound $\tup{\O,\Sigma}$-explanation of $\lambda$ is $\coNP$-hard.

    Finally, as for the perfect case, the proof of $\DP$-hardness is by a $\LOGSPACE$ reduction from the \emph{exact-4-colourability problem}, which is $\DP$-complete~\cite{Rothe03}. In particular, a graph $G$ is exact-4-colourable (i.e., 4-colourable and not 3-colourable) if and only if $\cert_{q_{G},J}^D=\{(c_4)\}$, i.e., if and only if $\cert_{q_{G},J}^D = \lambda$. So, checking whether the CQ $q_{G}$ is a perfect $\tup{\O,\Sigma}$-explanation of $\lambda$ is $\DP$-hard. \qed

\end{proof}

\begin{corollary}
    Complete-VQDEF($\dlliter$, GLAV, UCQ), Sound-VQDEF($\dlliter$, GLAV, UCQ), and Perfect-VQDEF($\dlliter$, GLAV, UCQ) are $\NP$-complete, $\coNP$-complete, and $\DP$-complete, respectively.
\end{corollary}

Finally, the lower bound proof of Theorem~\ref{th:vLB} can be easily adapted for the plain relational database case. Thus, given a schema $\S$, an $\S$-database $D$, a $D$-dataset $\lambda$, and a UCQ $q_{\S}$ over $\S$, it is $\DP$-complete the problem of deciding whether $q_{\S}$ defines $\lambda$ inside $D$ (the $\DP$ membership of this problem directly follows from Corollary~\ref{ch:UBP}).


\section{Computation}\label{sec:Computation}


In this section, we address the computation problem. We start by considering the case when the OBDM system $\Sigma$ at hand is inconsistent as a separate case. Given an inconsistent OBDM system $\Sigma=\tup{J,D}$ and a $D$-dataset $\lambda$ of arity $n$, we point out that any query $q_{\O}$ over the ontology $\O$ of the OBDM specification $J$ is the UCQ-minimally complete $\Sigma$-characterization of $\lambda$ (recall that the certain answers of any query $q_{\O}$ of arity $n$ \wrt an inconsistent OBDM system $\Sigma$ is the set of all possible $n$-tuples of constants occurring in $D$). Furthermore, if $\lambda=\dom(D)^n$, then any query $q_{\O}$ is also the UCQ-maximally sound (and therefore the perfect) $\Sigma$-characterization of $\lambda$; otherwise, i.e., $\lambda \subsetneq \dom(D)^n$, no sound (and therefore, no UCQ-maximally sound and no perfect) $\Sigma$-characterization of $\lambda$ exists.

Having thoroughly covered the case of inconsistent OBDM systems, in what follows in this section, unless otherwise stated, we implicitly assume to only deal with consistent OBDM systems.

Specifically, given a consistent OBDM system $\Sigma=\tup{J,D}$ and a $D$-dataset $\lambda$, we provide exponential time algorithms for computing UCQ-minimally complete and UCQ-maximally sound $\Sigma$-characterizations of $\lambda$, thus proving that, in this case, they always exist. As already observed in Proposition~\ref{prop:Uniq}, in our scenario all UCQ-minimally complete (resp., UCQ-maximally sound) characterizations of $\lambda$ are unique up to logical equivalence \wrt $\Sigma$, and therefore we refer to \emph{the} UCQ-minimally complete (resp., UCQ-maximally sound) $\Sigma$-characterization of $\lambda$.

Before illustrating the main techniques to compute such best characterizations, we provide two crucial properties about the canonical structure that we will use to establish the correctness of our algorithms.

\begin{proposition}\label{prop:homo}
Let $\Sigma=\tup{\tup{\O,\S,\M},D}$ be an OBDM system, $q_{\O}$ be a UCQ over $\O$, and $\vec{c}$ and $\vec{b}$ be two tuples of constants such that $(\C_{\O}^{\M(D)},\vec{c}) \rightarrow (\C_{\O}^{\M(D)},\vec{b})$. If $\vec{c} \in \cert_{q_{\O},J}^D$, then $\vec{b} \in \cert_{q_{\O},J}^D$.
\end{proposition}

\begin{proof}
    If $\Sigma$ is inconsistent, the claim is trivial. If $\Sigma$ is consistent, from Section~\ref{sec:Preliminaries} we know that $\vec{c} \in \cert_{q_{\O},J}^D$ implies the existence of a disjunct $q=\{\vec{x} \mid \exists \vec{y} \per \phi(\vec{x},\vec{y})\}$ in $q_{\O}$ for which $(\body(\phi),\vec{x}) \rightarrow (\C_{\O}^{\M(D)},\vec{c})$. Let $h$ be the homomorphism witnessing that $(\body(\phi),\vec{x}) \rightarrow (\C_{\O}^{\M(D)},\vec{c})$, and let $h'$ be the homomorphism witnessing that $(\C_{\O}^{\M(D)},\vec{c}) \rightarrow (\C_{\O}^{\M(D)},\vec{b})$, which holds by the premises of the proposition. The composition function $h''=h' \circ h$ is then a homomorphism witnessing that $(\body(\phi),\vec{x}) \rightarrow (\C_{\O}^{\M(D)},\vec{b})$. It follows that $\vec{b} \in \cert_{q_{\O},J}^D$, as required.\qed
\end{proof}

\begin{proposition}\label{prop:canhomo}
Let $\Sigma=\tup{\tup{\O,\S,\M},D}$ be a consistent OBDM system, $\vec{b}$ and $\vec{c}$ be two tuples of constants, and $q_{\vec{c}}$ be the CQ $q_{\vec{c}}=\query(\M(D),\vec{c})$. We have that $\vec{b} \in \cert_{q_{\vec{c}},J}^D$ if and only if $(\C_{\O}^{\M(D)},\vec{c}) \rightarrow (\C_{\O}^{\M(D)},\vec{b})$.
\end{proposition}

\begin{proof}
Suppose that $(\C_{\O}^{\M(D)},\vec{c}) \rightarrow (\C_{\O}^{\M(D)},\vec{b})$, and let $h$ be the homomorphism witnessing it. Consider the query $q_{\vec{c}}=\query(\M(D),\vec{c})=\{\vec{x} \mid \exists \vec{y} \per \phi(\vec{x},\vec{y})\}$. Observe that $\body(\phi)$ is obtained from $\M(D)$ by appropriately replacing each occurrence of each constant $c \in \dom(\M(D))$ either with a distinguished variable $x_c \in \vec{x}$ or with an existential variable $y_c \in \vec{y}$. This means that $h$ can be immediately transformed into a homomorphism witnessing that $(\body(\phi),\vec{x}) \rightarrow (\C_{\O}^{\M(D)},\vec{b})$, thus implying that $\vec{b} \in \cert_{q_{\vec{c}},J}^D$.

Suppose now that $\vec{b} \in \cert_{q_{\vec{c}},J}^D$. Since $\Sigma$ is consistent, it follows that there is a homomorphism $h$ witnessing that $(\body(\phi),\vec{x}) \rightarrow (\C_{\O}^{\M(D)},\vec{b})$, where $q_{\vec{c}}=\query(\M(D),\vec{c})=\{\vec{x} \mid \exists \vec{y} \per \phi(\vec{x},\vec{y})\}$. By considering again the relationship between $\body(\phi)$ and $\C_{\O}^{\M(D)}$, the homomorphism $h$ can be immediately transformed into a homomorphism $h'$ that witnesses $(\M(D),\vec{c}) \rightarrow (\C_{\O}^{\M(D)},\vec{b})$. It is now not hard to verify that $h'$ can be extended into a homomorphism $h''$ witnessing that $(\C_{\O}^{\M(D)},\vec{c}) \rightarrow (\C_{\O}^{\M(D)},\vec{b})$.
\end{proof}

We are now ready to present our techniques. We start with the complete case, and provide the algorithm~\hyperref[algo:mincomplete]{MinCompCharacterization} 
for computing UCQ-minimally complete characterizations. 

\begin{algorithm}[!htb]
    \caption{MinCompCharacterization}\label{algo:mincomplete}
    \begin{algorithmic}[1] 
    \Require 
    \Statex OBDM system $\Sigma=\langle J,D\rangle$ with $J=\tup{\O,\S,\M}$;
    \Statex $D$-dataset $\lambda = \{\vec{c_1}, \ldots, \vec{c_n}\}$
    \Ensure
    \Statex UCQ $q_\O$ over $\O$
    \Statex
        \State Compute $\M(D)$
        \State $q_\O \gets query(\M(D), \vec{c_1})$ $\cup \ldots \cup query(\M(D), \vec{c_n})$
        \State \Return $q_\O$
    \end{algorithmic}
\end{algorithm}

Informally, for each $\vec{c_i} \in \lambda$, the algorithm obtains from the set of atoms $\M(D)$ the CQ $\query(\M(D),\vec{c_i})$. Finally, the output is the union of all such CQs.

\begin{example}\label{ex:AlgoComplete}
    Let $J=\tup{\O,\S,\M}$ be the same OBDM specification of Example~\ref{ex:mainexample}. One can verify that for the $\S$-database $D=\{s_1(c_1),s_3(c_2,b),s_3(c_3,b)\}$ and the $D$-dataset $\lambda=\{(c_1),(c_2)\}$, MinCompCharacterization($\tup{J,D},\lambda$) returns the UCQ $q_{\O}=\query(\M(D),(c_1))$ $\cup$ $\query(\M(D),(c_2))$, where $\query(\M(D),(c_1))$ = $\{(x_{c_1}) \mid \exists y_{c_2},y_{c_3},y_b \per \textsf{Student}(x_{c_1}) \wedge \textsf{EnrolledIn}(y_{c_2},y_b) \wedge \textsf{EnrolledIn}(y_{c_3},y_b)\}$ and $\query(\M(D),(c_2))$ = $\{(x_{c_2}) \mid \exists y_{c_1},y_{c_3},y_b \per \textsf{EnrolledIn}(x_{c_2},y_b) \wedge \textsf{EnrolledIn}(y_{c_3},y_b) \wedge \textsf{Student}(y_{c_1})\}$. Furthermore, one can see that $q_{\O}$ is the UCQ-minimally complete $\Sigma$-characterization of $\lambda$, where $\Sigma=\tup{J,D}$.
\end{example}

The following theorem establishes termination and correctness of the~\hyperref[algo:mincomplete]{MinCompCharacterization} algorithm.

\begin{theorem}\label{th:CompuComp}
    MinCompCharacterization($\Sigma,\lambda$) terminates and returns the UCQ-minimally complete $\Sigma$-characterization of $\lambda$.
\end{theorem}
\begin{proof}
    Termination of the algorithm as well as completeness of the UCQ $q_{\O}$ returned are straightforward. 

    To prove that $q_{\O}$ is also the UCQ-minimally complete $\Sigma$-characterization of $\lambda$, it is enough to show that any query $q$ over $\O$ that is a complete $\Sigma$-characterization of $\lambda$ is such that $\cert_{q_{\O},J}^D \subseteq \cert_{q,J}^D$, where $\Sigma=\tup{J,D}$. We do this by contraposition. Let $q$ be
    a UCQ for which $\cert_{q_{\O},J}^D \not\subseteq \cert_{q,J}^D$, i.e., for a tuple of constants $\vec{b}$ we have $\vec{b} \not\in \cert_{q,J}^D$ but $\vec{b} \in \cert_{q_{\O},J}^D$.  
    This latter means that $\vec{b} \in \cert_{q_{\vec{c}},J}^D$ for some $q_{\vec{c}}=\query(\M(D),\vec{c})$ with $\vec{c} \in \lambda$. By Proposition~\ref{prop:canhomo}, one can see that $\vec{b} \in \cert_{q_{\vec{c}},J}^D$ implies $(\C_{\O}^{\M(D)},\vec{c}) \rightarrow (\C_{\O}^{\M(D)},\vec{b})$. By Proposition~\ref{prop:homo}, it follows that each UCQ $q'$ over $\O$ containing tuple $\vec{c}$ in its set of certain answers \wrt $\Sigma$ must contain also tuple $\vec{b}$ in such a set. Thus, since $\vec{b} \not\in \cert_{q,J}^D$, we derive that $\vec{c} \not\in \cert_{q,J}^D$ as well. Since $\vec{c} \in \lambda$, this latter clearly implies that $q$ is not a complete $\Sigma$-characterization of $\lambda$, as required.\qed
\end{proof}

We now turn to the sound case, and provide the algorithm~\hyperref[algo:maxsound]{MaxSoundCharacterization} 
for computing UCQ-maximally sound $\Sigma$-characterizations. 

\begin{algorithm}[!htb]
    \caption{MaxSoundCharacterization}\label{algo:maxsound}
    \begin{algorithmic}[1] 
    \Require 
    \Statex Consistent OBDM system $\Sigma=\langle J,D\rangle$ with $J=\tup{\O,\S,\M}$;
    \Statex $D$-dataset $\lambda = \{\vec{c_1}, \ldots, \vec{c_m}\}$ of arity $n$
    \Ensure
    \Statex UCQ $q_\O$ over $\O$
    \Statex
        \State $\lambda^- \gets \dom(D)^n \setminus \lambda$
        \State $q_\O \gets \{\vec{x} \mid \bot(\vec{x})\}$, where $\vec{x}=(x_1,\ldots,x_n)$
        \State Compute $\M(D)$
        \For {each $i \gets 1,\ldots,m$}
            \State $q_i \gets query(\M(D), \vec{{c_i}})$
            \If{$cert^D_{q_i, J} \cap \lambda^- = \emptyset$}
                \State $q_\O \gets q_\O \cup q_i$
            \EndIf
        \EndFor
        \State \Return $q_\O$
    \end{algorithmic}
\end{algorithm}

Intuitively, starting from the UCQ $\query(\M(D),\vec{c_1}) \cup \ldots \cup \query(\M(D),\vec{c_m})$, the algorithm simply discards all those disjuncts whose set of certain answers \wrt $\Sigma$ contain a tuple $\vec{b} \not \in \lambda$. We recall from Section~\ref{sec:Preliminaries} that the set of certain answers of a CQ $q_i$ \wrt a consistent OBDM system $\Sigma=\tup{J,D}$ can be computed by 
first obtaining its reformulation $\rewr_{q_i,J}$ over the source schema $\S$, and then by
evaluating this latter query directly over the $\S$-database $D$.

\begin{example}\label{ex:AlgoSound}
    Refer to Example~\ref{ex:AlgoComplete}. Since the certain answers of $\query(\M(D),(c_2))$ \wrt $\Sigma=\tup{J,D}$ include also $(c_3) \not \in \lambda$, MaxSoundCharacterization($\Sigma,\lambda$) returns the CQ $q_{\O}=\query(\M(D),(c_1))$, which is the UCQ-maximally sound $\Sigma$-characterization of $\lambda$.
\end{example}

The following theorem establishes termination and correctness of the~\hyperref[algo:maxsound]{MaxSoundCharacterization} algorithm.

\begin{theorem}\label{th:CompuSound}
    MaxSoundCharacterization($\Sigma, \lambda$) terminates and returns the UCQ-maximally sound $\Sigma$-characterization of $\lambda$.
\end{theorem}

\begin{proof}
    Termination of the algorithm as well as soundness of the UCQ $q_{\O}$ returned are straightforward. 

    To prove that $q_{\O}$ is also the UCQ-maximally sound $\Sigma$-characterization of $\lambda$, it is enough to show that any query $q$ over $\O$ that is a sound $\Sigma$-characterization of $\lambda$ is such that $\cert_{q,J}^D \subseteq \cert_{q_{\O},J}^D$, where $\Sigma=\tup{J,D}$. We do this by contraposition.
    Let $q$ be a UCQ for which $\cert_{q,J}^D \not\subseteq \cert_{q_{\O},J}^D$, i.e., for a tuple of constants $\vec{b}$ we have $\vec{b} \in \cert_{q,J}^D$ but $\vec{b} \not\in \cert_{q_{\O},J}^D$. If $\vec{b} \not \in \lambda$, then we immediately get that $q$ is not a sound $\Sigma$-characterization of $\lambda$, and we are done. So, assume that $\vec{b} \in \lambda$. Since $\vec{b} \not\in \cert_{q_{\O},J}^D$ and $\vec{b} \in \lambda$, it is easy to see that the algorithm discarded the disjunct $q_{\vec{b}}=\query(\M(D),\vec{b})$ (otherwise, we would trivially derive that $\vec{b} \in \cert_{q_{\vec{b}},J}^D$, and thus $\vec{b} \in \cert_{q_{\O},J}^D$, which is a contradiction to the fact that $\vec{b} \not\in \cert_{q_{\O},J}^D$). From the algorithm, one can see that the only reason $q_{\vec{b}}$ was discarded is because $\vec{g} \in \cert_{q_{\vec{b}},J}^D$ for at least a tuple $\vec{g} \not\in \lambda$ (i.e., $\vec{g} \in \dom(D)^n \setminus \lambda$).
    By Proposition \ref{prop:canhomo}, one can see that $\vec{g} \in \cert_{q_{\vec{b}},J}^D$ implies $(\C_{\O}^{\M(D)},\vec{b}) \rightarrow (\C_{\O}^{\M(D)},\vec{g})$. By Proposition~\ref{prop:homo}, it follows that each UCQ $q'$ over $\O$ containing tuple $\vec{b}$ in its set of certain answers \wrt $\Sigma$ must contain also tuple $\vec{g}$ in such a set. 
    Thus, since $\vec{b} \in \cert_{q,J}^D$, we derive that $\vec{g} \in \cert_{q,J}^D$ as well. Since $\vec{g} \not\in \lambda$, this latter clearly implies that $q$ is not a sound $\Sigma$-characterization of $\lambda$, as required.\qed
\end{proof}

Notice that, in all the cases in which a perfect characterization exists, it is clear that both the above algorithms return the same query $\query(\M(D),\vec{c_1}) \cup \ldots \cup \query(\M(D),\vec{c_n})$. As a direct consequence of both Theorem~\ref{th:CompuComp} and Theorem~\ref{th:CompuSound}, we get the following result.

\begin{corollary}\label{ch:CompuPerf}
    Either the UCQ $\query(\M(D),\vec{c_1}) \cup \ldots \cup \query(\M(D),\vec{c_n})$ is a perfect $\Sigma$-characterization of $\lambda=\{\vec{c_1},\ldots,\vec{c_n}\}$, or a perfect $\Sigma$-characterization of $\lambda$ in UCQ does not exist.
\end{corollary}

Furthermore, the combination of Corollary~\ref{ch:CompuPerf} and Proposition~\ref{prop:canhomo} allow us to provide a semantic test for the existence of perfect characterizations in UCQ in the OBDM case, which can be seen as the analogous of the semantic tests given in~\cite{BR17} and~\cite{Ortiz19} for the plain relational database case and the ontology-mediated query answering case, respectively. More specifically, given a consistent OBDM system $\Sigma=\tup{\tup{\O,\S,\M},D}$ and a $D$-dataset $\lambda$ of arity $n$, there exists a perfect $\Sigma$-characterization of $\lambda$ in UCQ if and only if it is the case that $(\C_{\O}^{\M(D)},\vec{c}) \not\rightarrow (\C_{\O}^{\M(D)},\vec{b})$ for each $\vec{c} \in \lambda$ and each $\vec{b} \in \dom(D)^n \setminus \lambda$.

In the next section, we study the computational complexity of the problem of deciding, given $\Sigma=\tup{\tup{\O,\S,\M},D}$ and $\lambda$, whether a perfect $\Sigma$-characterization of $\lambda$ exists.


\newcommand{\cl}{\mathit{cl}}

\section{Existence}\label{sec:Existence}

We now address the existence problem. For the scenario under consideration in this paper, the existence problem for both UCQ-minimally complete and UCQ-maximally sound characterizations is trivial, since by Theorems~\ref{th:CompuComp} and~\ref{th:CompuSound} they always exist. So, we only consider the perfect case, by defining a variant of the QDEF problem as defined in~\cite{Ortiz19}, where also a mapping in some mapping language is given as input.
\noindent
\begin{table}[h]
    \centering
    \framebox{
            \begin{tabular}{p{0.17\linewidth}p{0.68\linewidth}}
                \textsc{Problem}: & \textbf{QDEF($\L_{\O}$, $\L_{\M}$, $\Q$)} \\
                \textsc{Input}: & An OBDM system $\Sigma=\tup{\tup{\O,\S,\M},D}$ and a $D$-dataset $\lambda$, where $\O \in \L_{\O}$ and $\M \in \L_{\M}$.\\
                \textsc{Question}: & Is there a query $q_{\O} \in \Q$ over $\O$ such that $q_{\O}$ is the perfect $\Sigma$-characterization of $\lambda$?
            \end{tabular}
    }
\end{table}

In what follows, we show that the computational complexity of the above QDEF decision problem differs depending on the mapping language $\L_{\M}$ adopted. A key difference between GLAV and the special cases GAV and LAV is in the size of $\M(D)$. In GLAV mappings $\M(D)$ can be exponentially large due to the simultaneous presence of joins in the left-hand side, and existential variables in the right-hand side, of assertions (e.g., take $D=\{s_i(0),s_i(1) \mid 1 \leq i \leq n\}$ and $\M$ containing the GLAV assertion: $\{(x_1,\ldots,x_n) \mid s_1(x_1) \wedge \ldots \wedge s_n(x_n)\} \rightarrow \{(x_1,\ldots,x_n) \mid \exists y \per P(x_1,y) \wedge \ldots \wedge P(x_n,y)\}$). Conversely, in both LAV and GAV mappings, $\M(D)$ is always polynomially bounded since the former do not allow for joins in the left-hand side of assertions, whereas the latter do not allow for existential variables in the right-hand side of assertions and the arity of ontology predicates is fixed to at most $2$.

GAV and LAV mappings, however, differ for the effort in computing $\M(D)$. While in LAV mappings $\M(D)$ can be always computed in polynomial time, in GAV mappings there are CQs on the left-hand side of assertions, and so $\M(D)$ can not be computed in polynomial time (unless P=$\NP$). 

We start by characterizing the computational complexity of the simplest LAV case, then the GAV case, and finally the most general GLAV case. Interestingly, all the provided matching lower bounds hold even for fixed ontologies $\O=\emptyset$, i.e., ontologies without assertions, fixed $D$-dataset $\lambda$ containing a single unary tuple, and for both CQs and UCQs as query languages.

Importantly, for the scenario under consideration, due to Corollary~\ref{ch:CompuPerf}, the question in QDEF can be reformulated equivalently as follows: ``is $q_{\O}=\query(\M(D),\vec{c_1}) \cup \ldots \cup \query(\M(D),\vec{c_n})$ also a sound (and so, a perfect) $\Sigma$-characterization of $\lambda=\{\vec{c_1},\ldots,\vec{c_n}\}$?''.

\begin{theorem}
    \sloppy{QDEF($\dlliter$, LAV, UCQ) is $\coNP$-complete.}
\end{theorem}

\begin{proof}
    As for the membership in $\coNP$, we can first compute $\M(D)$ in polynomial time, and then, exactly as illustrated in Theorem~\ref{th:UBS}, we can check in $\coNP$ whether $\query(\M(D),\vec{c_1}) \cup \ldots \cup \query(\M(D),\vec{c_n})$ is also a sound (and so, a perfect) $\Sigma$-characterization of $\lambda=\{\vec{c_1},\ldots,\vec{c_n}\}$.

    $\coNP$-hardness directly follows from the plain relational database case~\cite{AnNS13}.
\end{proof}

Recall that the complexity class $\Tetwop$ has many characterizations: $\Tetwop = \text{P}^{\NP[\text{O}(\text{log } n)]}=$ P \emph{with a constant number of rounds of parallel queries to an} $\NP$ \emph{oracle}~\cite{BuHa91} (see also~\cite{Wagn90} for further characterizations).

\begin{theorem}
    QDEF($\dlliter$, GAV, UCQ) is $\Tetwop$-complete.
\end{theorem}

\begin{proof}
    As for the upper bound, for each pair of constants $(c_1,c_2) \in \dom(D)^2$ (resp., constant $c \in \dom(D)$) and for each atomic role $P$ (resp., concept $A$) in the alphabet of $\O$ we ask, all together with a single round of parallel queries to an $\NP$ oracle, whether $P(c_1,c_2) \in \M(D)$ (resp., $A(c) \in \M(D)$). Then, with a second and final round, due to Theorem~\ref{th:UBS}, we can ask with a single query to an $\NP$ oracle whether $\query(\M(D),\vec{c_1}) \cup \ldots \cup \query(\M(D),\vec{c_n})$ is also a sound (and so, a perfect) $\Sigma$-characterization of $\lambda=\{\vec{c_1},\ldots,\vec{c_n}\}$.

    As for the lower bound, the proof of $\Tetwop$-hardness is by a $\LOGSPACE$ reduction from the \emph{odd clique problem}, which is $\Tetwop$-complete~\cite{Wagn87}. \emph{Odd clique} is the problem of deciding, given a finite and undirected graph without loops $G=(V,E)$, whether the maximum clique size of $G$ is an odd number. Without loss of generality, we may assume that $E$ contains at least an edge and that the cardinality of $V$ is an even number (indeed, it is always possible to add fresh isolated nodes to the graph $G$ without changing its maximum clique size). 

    Let $V=\{v_1,\ldots,v_n\}$, we define an OBDM system $\Sigma_G=\tup{J_G,D_G}$ as follows: $J_G=\tup{\O,\S_G,\M_G}$ is an OBDM specification such that $\O=\emptyset$, $\S_G=\{e,s_1,\ldots,s_n\}$, and $\M_G$ has the following GAV assertions, for each possible odd number $i \in [1,n]$:
    \begin{align*}
        & \{(x) \mid \exists y_1,\ldots, y_i \per s_i(x) \wedge \cl_i\} \rightarrow \{(x) \mid A_i(x)\},\\
        & \{(x) \mid \exists y_1,\ldots, y_{i+1} \per s_{i+1}(x) \wedge \cl_{i+1}\} \rightarrow \{(x) \mid A_i(x)\},
    \end{align*}
    where $A_i$ is an atomic concept in the alphabet of $\O$, and, for each $p \in [1,n]$: 
    $$\cl_p=\bigwedge_{\{(k,j) \mid 1 \leq k<j \leq p\}} e(y_k,y_j).
    $$ Intuitively, $\cl_p$ asks whether $G$ contains a clique of size $p$. Finally, $D_G=\{e(x_1,x_2) \mid (x_1,x_2) \in E\} \cup \{e(x_2,x_1) \mid (x_1,x_2) \in E\} \cup \{s_i(c) \mid 1 \leq i \leq n \text{ and } i \text{ is odd}\} \cup \{s_i(c') \mid 2 \leq i \leq n \text{ and } i \text{ is even}\}$. Let, moreover, $\lambda$ be the fixed $D_G$-dataset $\lambda=\{(c)\}$.
    
    Notice that $\lambda$ is fixed, whereas the OBDM system $\Sigma_G$ can be constructed in $\LOGSPACE$ from an input graph $G$. 
    
    The correctness of the reduction is mainly based on the following property:

    \begin{claim}
        Let $i \in [1,n]$ be an odd number. We have that:
        \begin{enumerate}
            \item \label{it:Odd} $A_i(c) \in \C_{\O}^{\M_{G}(D_{G})}$ if and only if $G$ contains a clique of size $i$.
            \item \label{it:Even} $A_i(c') \in \C_{\O}^{\M_{G}(D_{G})}$ if and only if $G$ contains a clique of size $i+1$.
        \end{enumerate}
    \end{claim}

    \begin{proof}
        As for~\ref{it:Odd}, since $s_i(c) \in D_G$, it is easy to see that the query $q_i=\{(x) \mid \exists y_1,\ldots, y_i \per s_i(x) \wedge \cl_i\}$ is such that $(c) \in q_i^{D_G}$ if and only if $G$ has a clique of size $i$. Thus, due to the GAV assertion $q_i \rightarrow \{(x) \mid A_i(x)\}$ occurring in $\M_G$, we have $A_i(c) \in \C_{\O}^{\M_G(D_G)}$ if and only if $G$ has a clique of size $i$.

        As for~\ref{it:Even}, since $s_{i+1}(c') \in D_G$, it is easy to see that the query $q_{i+1}=\{(x) \mid \exists y_1,\ldots, y_{i+1} \per s_{i+1}(x) \wedge \cl_{i+1}\}$ is such that $(c') \in q_{i+1}^{D_G}$ if and only if $G$ has a clique of size $i+1$. Thus, due to the GAV assertion $q_{i+1} \rightarrow \{(x) \mid A_i(x)\}$ occurring in $\M_G$, if $G$ has a clique of size $i+1$, then $A_i(c') \in \C_{\O}^{\M_G(D_G)}$. Conversely, suppose that $G$ has not a clique of size $i+1$. On the one hand, the assertion $q_{i+1} \rightarrow \{(x) \mid A_i(x)\}$ does not make $A_i(c')$ true in $\C_{\O}^{\M_G(D_G)}$. On the other hand, since $s_i(c') \not \in D_G$, not even the assertion $\{(x) \mid \exists y_1,\ldots, y_{i} \per s_{i}(x) \wedge \cl_{i}\} \rightarrow \{(x) \mid A_i(x)\}$ makes $A_i(c')$ true in $\C_{\O}^{\M_G(D_G)}$. Thus, $A_i(c') \not\in \C_{\O}^{\M_G(D_G)}$. \qed
    \end{proof}

    With the above property at hand, we can now prove that the maximum clique size of a graph $G$ is an odd number if and only if the CQ $q_{\O}=\query(\M_{G}(D_{G}),c)$ is also a sound (and so, a perfect) $\Sigma_{G}$-characterization of $\lambda$, thus showing the claimed lower bound.

    ``\textbf{Only-if part:}'' Suppose that the maximum clique size of $G$ is $p$, where $p$ is an odd number. Due to the above claim, we have that $\C_{\O}^{\M_G(D_G)}=\{A_1(c),A_1(c'),A_3(c),A_3(c'),\ldots, A_p(c)\}$ (observe that $A_p(c') \not\in \C_{\O}^{\M_G(D_G)}$ because $G$ has not a clique of size $p+1$ by assumption), and so $q_{\O}=\query(\M_{G}(D_{G}),c)=\{(x_c) \mid \exists y_{c'} \per \phi_{\O}(x_c,y_{c'})\}$, where $\phi_{\O}(x_c,y_{c'})$ = $A_1(x_c) \wedge A_1(y_{c'}) \wedge A_3(x_c) \wedge A_3(c') \wedge \ldots \wedge A_p(x_c)$. It is straightforward to verify that $(\body(\phi_{\O}),(x_c)) \rightarrow (\C_{\O}^{\M_G(D_G)},(c))$ but $(\body(\phi_{\O}),(x_c)) \not\rightarrow (\C_{\O}^{\M_G(D_G)},(c'))$. It follows that $\cert_{q_{\O},J_G}^{D_G}=\{(c)\}$, i.e., $q_{\O}$ is a perfect $\Sigma_G$-characterization of $\lambda$.

    ``\textbf{If part:}'' Suppose that the maximum clique size of $G$ is $r$, where $r$ is an even number. Due to the above claim, we have that $\C_{\O}^{\M_G(D_G)}=\{A_1(c),A_1(c'),A_3(c),A_3(c'),\ldots, A_{r-1}(c),A_{r-1}(c')\}$ (observe that $A_{r-1}(c') \in \C_{\O}^{\M_G(D_G)}$ and $A_{r+1}(c) \not\in \C_{\O}^{\M_G(D_G)}$ because by assumption $G$ has a clique of size $r$ but not of size $r+1$), and so $q_{\O}=\query(\M_{G}(D_{G}),c)=\{(x_c) \mid \exists y_{c'} \per \phi_{\O}(x_c,y_{c'})\}$, where $\phi_{\O}(x_c,y_{c'})$ = $A_1(x_c) \wedge A_1(y_{c'}) \wedge A_3(x_c) \wedge A_3(c') \wedge \ldots \wedge A_{r-1}(x_c) \wedge A_{r-1}(y_{c'})$. It is straightforward to verify that $(\body(\phi_{\O}),(x_c)) \rightarrow (\C_{\O}^{\M_G(D_G)},(c'))$. It follows that $(c') \in \cert_{q_{\O},J_G}^{D_G}$, i.e., $q_{\O}$ is not a sound (and so, not a perfect) $\Sigma_G$-characterization of $\lambda$. \qed
\end{proof}

\begin{theorem}
    QDEF($\dlliter$, GLAV, UCQ) is $\coNEXPTIME$-complete.
\end{theorem}
\begin{proof}
    We start by discussing the upper bound. We show how to check whether $q_{\O}=\query(\M(D),\vec{c_1}) \cup \ldots \cup \query(\M(D),\vec{c_m})$ is not a sound (and so, not a perfect) $\Sigma$-characterization of $\lambda=\{\vec{c_1},\ldots,\vec{c_m}\}$ in $\NEXPTIME$.
    
    As a first step, we compute $q_{\O}=\query(\M(D),\vec{c_1}) \cup \ldots \cup \query(\M(D),\vec{c_m})$ in exponential time (note that $\M(D)$ can be exponentially large, and so also the UCQ $q_{\O}$). Then, we can proceed similarly as in the proof of Theorem~\ref{th:UBS}. We guess \myi a tuple of constants $\vec{c}$, and \myii $q'_{\O}$, $\rho_{\O}$, $q_{\S}$, $\rho_{\M}$, and $f$ (which now can be objects of exponential size). Finally, we check in exponential time whether \myi $\vec{c} \in \dom(D)^n \setminus \lambda$, where $n$ is the arity of the tuples in $\lambda$, and \myii the following condition holds: $\vec{c} \in \cert_{q_{\O},J}^D$ or $\Sigma$ is inconsistent.
    
    As for the lower bound, the proof of $\coNEXPTIME$-hardness is by a polynomial time reduction from the \emph{complement of the succinct clique problem}. The \emph{succinct clique problem} is known to be $\NEXPTIME$-complete~\cite{PaYa86}. here we show that QDEF($\dlliter$, GLAV, UCQ) is $\coNEXPTIME$-hard. 

    To simplify the readability of the proof, we first illustrate the main idea behind it. In particular, as a first step we provide an alternative proof of $\coNP$-hardness of UCQ-definability in the plain relational database case. The $\coNP$-hardness proof is by a $\LOGSPACE$ reduction from the \emph{complement of the clique problem}. We have a schema $\S=\{e,P\}$. Let $\tup{G,k}$ be an instance of the \emph{clique problem}, where $G=(V,E)$ is a finite undirected graph without loops and with $E\neq \emptyset$, and $k$ is a natural number (written in unary). Starting from $G$, we define an $\S$-database $D=D_k \cup D_G$ as follows: 
    \begin{align*}
        & D_k = \{e(y_i,y_j) \mid 1 \leq i < j \leq k\} \cup \{P(c,y_i) \mid 1 \leq i \leq k\} \\
        & D_G = \{e(v_i,v_j) \mid (v_i,v_j) \in E\} \cup \{e(v_j,v_i) \mid (v_i,v_j) \in E\} \cup \{P(c',v) \mid v \in V\}
    \end{align*}
    Finally, $\lambda$ is the fixed $D$-dataset $\{(c)\}$. Intuitively, $D_k$ represents a clique of size $k$, whereas $D_G$ represents the graph $G$.

    Notice that $D$ can be constructed in $\LOGSPACE$ from $G$. Let $q_k=\query(D_k,(c))$, which is equivalent to:
    $$\{(x) \mid \exists y_1,\ldots,y_k \per P(x,y_1) \wedge \ldots \wedge P(x,y_k) \wedge \bigwedge_{1 \leq i < j \leq k}(e(y_i,y_j))\}.
    $$

    One can easily verify that $\lambda$ is UCQ-definable inside $D$ if and only if $q_k^D=\{(c)\}$. Indeed, $q_k$ is similar to the canonical UCQ $\query(D,(c))$~\cite{BaRo17}, where, however, only constants \emph{reachable} from $c$ are taken into account in the query. More specifically, by construction, the evaluation of $q_k$ over $D$ is either $\{(c)\}$, or $\{(c),(c')\}$. Indeed, $q_k^D$ always contain the tuple $(c)$, and it contains the tuple $(c')$ if and only if $(D_k,(c)) \rightarrow (D_G,(c'))$ (obviously, $(D_k,(c)) \rightarrow (D_G,(c'))$ if and only if $(\body(\phi),(x)) \rightarrow (D,(c'))$, where $\phi$ is the body of $q_k$). In the next property, we are going to show that this latter is the case if and only if $G$ has a clique of size $k$, thus proving the claimed $\coNP$-hardness.

    \begin{claim}
        $G$ has a clique of size $k$ if and only if $(D_k,(c)) \rightarrow (D_G,(c'))$. 
    \end{claim}

    \begin{proof}
        ``\textbf{Only-if part:}'' Suppose $G$ has a clique of size $k$, i.e., there are $k$ nodes in $G$ forming a clique. This immediately implies the existence of a homomorphism $h$ from $\dom(D_k)$ to $\dom(D_G)$ mapping \myi $c$ to $c'$, and \myii constant $y_i$ to a constant $v_i$ representing a node in such a clique, for each possible $i \in [1,k]$. Thus, we have $(D_k,(c)) \rightarrow (D_G,(c'))$, as required.

        ``\textbf{If part:}'' Suppose $(D_k,(c)) \rightarrow (D_G,(c'))$, i.e., there is a homomorphism $h$ from $\dom(D_k)$ to $\dom(D_G)$ with $h(c)=c'$. This immediately implies that the set of facts $\{e(h(y_i),h(y_j)) \mid e(y_i,y_j) \in D_k\}$ that must occur in $D_G$ due to the assumption that $h$ is a homomorphism from $\dom(D_k)$ to $\dom(D_G)$ denotes a clique of size $k$ inside the graph $G$. \qed
    \end{proof}
    
    We are now ready to come back to show that QDEF($\dlliter$, GLAV, UCQ) is $\coNEXPTIME$-hard. The proof of $\coNEXPTIME$-hardness is by a polynomial time reduction from the \emph{complement of the succinct clique problem}. Given a succinct representation of a graph $C_G$ representing a finite undirected graph $G=(V,E)$ without loops and with $E \neq \emptyset$, and given a natural number $k$ in unary, \emph{succinct clique} is the problem of deciding whether the graph represented by $C_G$ has a clique of size $k$. \emph{Succinct clique} is known to be $\NEXPTIME$-complete~\cite{PaYa86}. 

    For a succinct representation $C_G$ of a graph $G=(V,E)$ with $m$ nodes, without loss of generality, we implicitly mean a circuit using $2\cdot b$ input gates $\vec{x}=(x_1,\ldots,x_b,x_{b+1},\ldots,x_{2\cdot b})$ (where $2^b=m$) for which on input $(a_1,\ldots,a_b,a_{b+1},\ldots,a_{2\cdot b})$ circuit $C_G$ outputs true if and only if the two nodes $v_i,v_j$ in $V$ represented by $v_i=(a_1,\ldots,a_b)$ and $v_j=(a_{b+1},\ldots,a_{2\cdot b})$ are such that $(v_i,v_j) \in E$ (see~\cite{GaWi83} for more details). Moreover, from a circuit $C_G$ with $\vec{x}$ as input gates, we denote by $F_{C_G}(\vec{x},\vec{w})$ the 3-CNF formula obtained by applying the \emph{Tseitin transformation}~\cite{Tsei83}, where $\vec{w}$ are the fresh variables introduced by the transformation. We recall that such transformation is linear in the size of $C_G$. And, among the introduced variables $\vec{w}$ introduced by the linear transformation, there is one variable in $\vec{w}$, denoted by $w$, which represents the output gate of the circuit. More formally, the transformation is such that if on input $\vec{a}=(a_1,\ldots,a_b,a_{b+1},\ldots,a_{2\cdot b})$ circuit $C_G$ outputs true, there there is exactly one satisfying assignment of formula $F_{C_G}(\vec{a},\vec{w})$ with $1$ as truth assignment to variable $w$, otherwise (i.e., $C_G$ outputs false on $\vec{a}$) there is no satisfying assignment of $F_{C_G}(\vec{a},\vec{w})$ with $1$ as truth assignment to variable $w$.

    Let $\tup{C_G,k}$ be an instance of the succinct clique problem, where $C_G$ is a circuit with $2 \cdot b$ input gates succinctly representing a graph $G=(V,E)$ of $m=2^{b}$ nodes, and $k$ is a natural number written in unary. Let $F_{C_G}(\vec{x},\vec{w})=p_1 \wedge \ldots \wedge p_r$ be the 3-CNF formula associated to $C_G$, where each clause $p_i$ is a disjunction of three literals, each literal being either a variable in $\vec{x} \cup \vec{w}$ or its negated. For $i \in [1,p]$, we denote by $o_{i_1},o_{i_2},o_{i_3}$ the first, the second, and the third, respectively, variable appearing (either positive or negated) in clause $p_i$. 
    
    Starting from $\C_G$, we define an OBDM system $\Sigma=\tup{J,D}$ and a $D$-dataset $\lambda$ as follows, where $J=\tup{\O,\S,\M}$ is an OBDM specification and $D$ is an $\S$-database. First, $\O=\emptyset$ is an empty set of assertions containing binary predicates $e$, $P$, and $V_i$ for each $i \in [1,b]$. Intuitively, as in the reduction from normal clique, $e$ denotes the edges of graphs and $P$ connects constants $c$ and $c'$ to nodes of the $k$-clique graph and the input graph $G$, respectively. The additional predicates $V_i$'s are used to encode nodes of the graphs. Second, $\S=\S_k \cup \S_G$, where $\S_k=\{s_e,s_p,s_v\}$ and $\S_G=\{s',s_w,p_1,\ldots,p_r\}$ with $s'$ and $s_w$ unary predicates, $s_e$ and $s_p$ binary predicates, $p_1$, $p_2$, $\ldots$, and $p_r$ ternary predicates, and $s_v$ a $b+1$ predicate. Third, the $\S$-database $D=D_k \cup D_G$ is as follows. For each $i \in [1,k]$, $D_k$ contains \myi $k-1$ constants $y_i^1,\ldots,y_i^{k-1}$ used to represent the clique without ever repeating the same constant, and \myii $b$ constants $d_i^1, \ldots, d_i^b$ used to encode node $y_i$ and ensuring that the distinct constants $y_i^1,\ldots,y_i^{k-1}$ actually denote the same node $y_i$.
    \begin{align*}
        D_k = & \{s_e(y_i^{j},y_l^i) \mid 1 \leq i \leq j < k \text{ and }l=j+1\} \cup \\
        & \{s_p(c,y_i^l) \mid 1 \leq i \leq k \text{ and } 1 \leq l < k\} \cup \\
        & \{s_v(y_i^l,d_i^1,\ldots,d_i^b) \mid 1 \leq i \leq k \text{ and } 1 \leq l < k\}
    \end{align*}

    For example, if $k=4$ and $b=2$, then we have $D_k=\{e(y_1^1,y_2^1),e(y_1^2,y_3^1),e(y_1^3,y_4^1),e(y_2^2,y_3^2),e(y_2^3,y_4^2),e(y_3^3,y_4^3)\}\cup \{s_p(c,y_i^1) \mid 1 \leq i \leq 4\}\cup \{s_p(c,y_i^2) \mid 1 \leq i \leq 4\}\cup \{s_p(c,y_i^3) \mid 1 \leq i \leq 4\}\cup \{s_v(y_i^1,d_i^1,d_i^2) \mid 1 \leq i \leq 4\}\cup \{s_v(y_i^2,d_i^1,d_i^2) \mid 1 \leq i \leq 4\}\cup \{s_v(y_i^3,d_i^1,d_i^2) \mid 1 \leq i \leq 4\}$. Informally, $e$ contains a clique of size $k=4$ without ever repeating a node, while relation $s_v$ ensures that, for each $i \in [1,4]$, the distinct constants $y_i^1, y_i^2,y_i^3$ actually denote the same node $y_i$ because they are connected with same elements $d_i^1,d_i^2$ in $s_v$.

    As for $D_G$, we have $s'(c') \in D_G$ and $s_w(1) \in D_G$. Observe that for each clause $p_i$ there are exactly seven satisfying truth assignments of the clause, each of the form $(t_{j_1}^i,t_{j_2}^i,t_{j_3}^i)$, where $j \in [1,7]$ and each $t_{j_l}^i$ is the truth assignment (i.e., either constant $0$ or constant $1$) given to the variable $o_{i_l}$ by $j$. The predicate $p_i$ in $D_G$ associated to the clause simply list such satisfying assignments: 
    $$
        D_G=\{p_i(t_{j_1}^i,t_{j_2}^i,t_{j_3}^i) \mid 1 \leq i \leq r \text{ and } 1 \leq j \leq 7\} \cup \{s'(c')\} \cup \{s_w(1)\}
    $$ 

    For example, if the circuit $C_G$ succinctly representing $G$ corresponds to formula $x_1\text{ } \mathsf{XOR}\text{ } x_2$ (note $b=1$), then the 3-CNF formula is $F_{C_G}=\exists x_1,x_2,w \per (\neg x_1 \vee \neg x_2 \vee \neg w) \wedge (x_1 \vee x_2 \vee \neg w) \wedge (x_1 \vee \neg x_2 \vee w) \wedge (\neg x_1 \vee x_2 \vee w)$ (with $w$ the output gate being the only fresh variable introduced by the Tseitin transformation) and $D_G$ is as follows:
    \begin{align*}
    D_G=&\{s'(c'), s_w(1),\\
    & p_1(0,0,0), p_1(0,0,1), p_1(0,1,0), p_1(0,1,1), p_1(1,0,0), p_1(1,0,1), p_1(1,1,0), \\ 
    & p_2(0,0,0), p_2(0,1,0), p_2(0,1,1), p_2(1,0,0), p_2(1,0,1), p_2(1,1,0), p_2(1,1,1), \\ 
    & p_3(0,0,0), p_3(0,0,1), p_3(0,1,1), p_3(1,0,0), p_3(1,0,1), p_3(1,1,0), p_3(1,1,1) \\
    & p_4(0,0,0), p_4(0,0,1), p_4(0,1,0), p_4(0,1,1), p_4(1,0,1), p_4(1,1,0), p_4(1,1,1)\}
    \end{align*}

    The fixed $D$-dataset is $\lambda=\{(c)\}$. It remains to describe the GLAV mapping $\M$ in the OBDM system $\Sigma$. We have $\M=\M_k \cup \M_G$, where $\M_k=\{m_k^1,m_k^2,m_k^3\}$ and $\M_G=\{m_G\}$ such that $\M_k$ is simply as follows:
    \begin{align*}
        m^1_k: & \{(x',x'') \mid s_e(x',x'')\} \rightarrow  \{(x',x'') \mid e(x',x') \}\\
        m^2_k: & \{(x,x') \mid s_p(x,x')\} \rightarrow \{(x,x') \mid P(x,x')\}\\
        m^3_k: & \{(x,x_1,\ldots,x_b) \mid s_v(x,x_1,\ldots,x_b)\} \rightarrow \\ & \{(x,x_1,\ldots,x_b) \mid V_1(x,x_1) \wedge \ldots \wedge V_b(x,x_b)\}
    \end{align*}
    and $m_G \in \M_G$ is the following GLAV assertion:
    \begin{align*}
        & \{(x,x_1,\ldots,x_{2\cdot b}) \mid \exists \vec{w} \per s'(x) \wedge p_1(o_{1,1},o_{1,2},o_{1,3}) \wedge \ldots \wedge p_r(o_{r,1},o_{r,2},o_{r,3}) \wedge s_w(w)\} \rightarrow\\
        & \{(x,x_1,\ldots,x_{2\cdot b}) \mid \exists z_1,z_2 \per e(z_1,z_2) \wedge e(z_2,z_1) \wedge P(x,z_1) \wedge P(x,z_2) \wedge \\ & \qquad\qquad\qquad\qquad\qquad\quad V_1(z_1,x_1) \wedge \ldots \wedge V_b(z_1,x_b) \wedge \\ & \qquad\qquad\qquad\qquad\qquad\quad V_1(z_2,x_{b+1}) \wedge \ldots \wedge V_b(z_2,x_{2\cdot b})\}
    \end{align*}

    For example, for the circuit $C_G$ given before, the GLAV assertion $m_G$ is:
    \begin{align*}
        & \{(x,x_1,x_2) \mid \exists w_1 \per s'(x) \wedge p_1(x_1,x_2,w) \wedge p_2(x_1,x_2,w) \wedge \\ & \qquad\qquad\qquad\qquad\qquad p_3(x_1,x_2,w) \wedge p_4(x_1,x_2,w) \wedge s_w(w)\} \rightarrow\\
        & \{(x,x_1,x_2) \mid \exists z_1,z_2 \per e(z_1,z_2) \wedge e(z_2,z_1) \wedge P(x,z_1) \wedge P(x,z_2) \wedge \\ & \qquad\qquad\qquad\quad\quad\quad V_1(z_1,x_1) \wedge V_1(z_2,x_2)\}
    \end{align*}

    Observe that $\C_{\O}^{\M(D)}=\M(D)=\M_k(D_k) \cup \M_G(D_G)$. Informally, the extension of predicate $e$ in $\M_k(D_k)$ describes a $k$-clique graph without ever repeating a node, while the extensions of predicates $V_i$'s in $\M_k(D_k)$ ensure that, for each $i \in [1,k]$, the distinct constants $y_i^1, \ldots, y_i^{k-1}$ actually denote the same node $y_i$ because they are connected with same elements $d_i^1,\ldots,d_i^b$. Finally, the extension of predicate $P$ in $\M_k(D_k)$ simply contains $(c,y_i^l)$ for each constant $y_i^l$ occurring in the extension of $e$ in $\M_k(D_k)$. 
    
    As for $\M_G(D_G)$, let $\vec{a}=(a_1,\ldots,a_b,a_{b+1},a_{2 \cdot b})$ be an input to the circuit $C_G$. By construction, one can easily verify that $C_G$ outputs true if and only if $(c,a_1,\ldots,a_b,a_{b+1},a_{2 \cdot b})$ is a tuple in the evaluation of the left-hand side query of $m_G$ over $D_G$. Thus, for each edge $((a_1,\ldots,a_b),(a_{b+1},a_{2 \cdot b})) \in E$ represented by circuit $C_G$, the chase $\M_G(D_G)$ produces two fresh variables $z_1$ and $z_2$ that simulates, respectively, $v_1=(a_1,\ldots,a_b)$ and $v_2=(a_{b+1},\ldots,a_{2\cdot b})$ and connect them trough predicate $e$ (simulating so the edge $(v_1,v_2)\in E$ described by circuit $C_G$). Moreover, both the freshly introduced variables $z_1$ and $z_2$ are connected trough the extensions of predicates $V_i$'s in the following way: $V_1(z_1,a_1),\ldots,V_b(z_1,a_b),V_1(z_2,a_{b+1}),\ldots,V_b(z_2,a_{2 \cdot b})$. This ensures that two distinct variables $z$ and $z'$ that actually denote the same node in $G$, denote the same node also in $\M_G(D_G)$  because they are connected with same elements. Finally, the extension of predicate $P$ in $\M_G(D_G)$ simply contains $(c',z)$ for each freshly introduced variable $z$ by the application of the chase. 

    Notice that $\lambda=\{(c)\}$ is a fixed $D$-dataset and both the OBDM specification $J=\tup{\O,\S,\M}$ and the $\S$-database $D$ can be always constructed in polynomial time from $C_G$. Let $q_k=\query(\M(D_k),(c))$, which is equivalent to:
    \begin{align*}
        \{(x) \mid & \exists y_1^1,\ldots,y_1^{k-1},\ \ldots\ , y_k^1,\ldots,y_k^{k-1},d_1^1,\ldots,d_1^b,\ \ldots\ ,d_k^1,\ldots,d_k^b \per \\
        & \bigwedge_{1 \leq i \leq j < k \text{ and }l=j+1} (e(y_i^{j},y_l^i)) \wedge \bigwedge_{1 \leq i \leq k \text{ and } 1 \leq l < k} (P(x,y_i^l)) \wedge \\
        & \bigwedge_{1 \leq i \leq k \text{ and } 1 \leq l < k} (V_1(y_i^l,d_i^1) \wedge \ldots \wedge V_b(y_i^l,d_i^b)) \}
    \end{align*}

    One can easily verify that a UCQ-perfect $\Sigma$-characterization of $\lambda$ exists if and only if $\cert_{q_k,J}^D=\{(c)\}$. Indeed, $q_k$ is similar to $\query(\M(D),(c))$ (i.e., the UCQ-minimally complete $\Sigma$-characterization of $\lambda$), where, however, only elements \emph{reachable} from $c$ in $\M(D)$ are taken into account in the query. More specifically, by construction, the set of certain answers of $q_k$ with respect to $\Sigma=\tup{J,D}$ is either $\{(c)\}$, or $\{(c),(c')\}$. Indeed, $\cert_{q_k,J}^D$ always contain the tuple $(c)$, and it contains the tuple $(c')$ if and only if $(\M_k(D_k),(c)) \rightarrow (\M_G(D_G),(c'))$. Clearly, $(\M_k(D_k),(c)) \rightarrow (\M_G(D_G),(c'))$ if and only if $(\body(\phi),(x)) \rightarrow (\C_{\O}^{\M(D)},(c'))$, where $\phi$ is the body of $q_k$ (recall that this latter is equivalent to $c' \in \cert_{q_k,J}^D$). In the next property, we are going to show that $(\M_k(D_k),(c)) \rightarrow (\M_G(D_G),(c'))$ if and only if the graph $G$ represented by the circuit $C_G$ has a clique of size $k$, thus proving the claimed $\coNEXPTIME$-hardness and concluding the proof.

    \begin{claim}
        The graph $G$ represented by circuit $C_G$ has a clique of size $k$ if and only if $(\M_k(D_k),(c)) \rightarrow (\M_G(D_G),(c'))$. 
    \end{claim}

    \begin{proof}
        ``\textbf{Only-if part:}'' Suppose the graph $G$ represented by circuit $C_G$ has a clique of size $k$. Let $(v_1,v_2,\ldots,v_k)$ be the nodes forming such clique in $G$, where each node $v_i$ is encoded by $v_i=(a_{i,1},\ldots,a_{i,b})$ in $C_G$ (each $a_{i,x}$ is either $0$ or $1$). Consider any edge $(v_i,v_l) \in E$ (or $(v_l,v_i) \in E$) which by assumption exists for each pair $(i,l)$ with $1 \leq i < l \leq k$. Due to the GLAV assertion $m_G \in \M_G$, by construction, $\M_G(D_G)$ introduces two fresh variables (without loss of generality, let denote them by $z_i^j$ and $z_l^i$, where $j=l-1$) such that \myi $\{V_1(z_i^j,a_{i,1}), \ldots, V_b(z_i^j,a_{i,b})\} \subseteq \M_G(D_G)$, \myii $\{V_1(z_{l}^i,a_{l,1}), \ldots, V_b(z_{l}^i,a_{l,b})\} \subseteq \M_G(D_G)$, and \myiii $\{e(z_i^j,z_l^i), e(z_l^i,z_i^j)\} \subseteq \M_G(D_G)$. Observe that, for each node $v_i$, in this way there are $k-1$ fresh variables $z_{i}^1,z_{i}^2,\ldots,z_{i}^{k-1}$ representing $v_i$, i.e., $\{V_1(z_i^j,a_{i,1}), \ldots, V_b(z_i^j,a_{i,b})\} \subseteq \M_G(D_G)$ for each $j \in [1,k-1]$.
        
        Consider now the function $h$ from $\dom(\M_k(D_k))$ to $\dom(\M_G(D_G))$ such that \myi $h(c)=c'$, \myii $h(y_i^l)=z_i^l$ for each $i \in [1,k]$ and for each $l \in [1,k-1]$, and \myiii $h(d_i^j)=a_{i,j}$ for each $i \in [1,k]$ and for each $j \in [1,b]$. It is straightforward to verify that $h$ is a homomorphism witnessing that $(\M_k(D_k),(c)) \rightarrow (\M_G(D_G),(c'))$, as required.

        ``\textbf{If part:}'' Suppose $(\M_k(D_k),(c)) \rightarrow (\M_G(D_G),(c'))$, i.e., there is a homomorphism $h$ from $\dom(\M_k(D_k))$ to $\dom(\M_G(D_G))$ with $h(c)=c'$. Observe that $(h(d_i^1),\ldots,h(d_i^b))$ is the encoding of a node $v_i$ in circuit $C_G$ by construction, for each possible $i \in [1,k]$. Furthermore, for each $i \in [1,k]$, ontology predicates $V_1,\ldots,V_b$ ensure that elements $y_i^1, \ldots, y_i^{k-1}$ are mapped to distinct fresh variables introduced by $\M_G(D_G)$ that actually denote the same node, because it must be the case that $V_1(h(y_i^1),h(d_i^1)) \in \M_G(D_G), \ldots, V_b(h(y_i^1),h(d_i^b)) \in \M_G(D_G),\ \ldots,\ V_1(h(y_i^{k-1}),h(d_i^1)) \in \M_G(D_G), \ldots, V_b(h(y_i^{k-1}),h(d_i^b)) \in \M_G(D_G)$. Now, it is easy to verify that if in $\M_G(D_G)$ we replace each variable $z$ introduced by $\M_G(D_G)$ with the value $v=(a_1,\ldots,a_b)$ for which $V_1(z,a_1) \in \M_G(D_G), \ldots, V_b(z,a_b)\in \M_G(D_G)$, then by looking at the extension of $e$ we obtain exactly the graph $G$ represented by circuit $C_G$.
        
        From the above observations, and the fact that the set of facts $\{e(h(y_i^{j}),h(y_l^i)) \mid 1 \leq i \leq j < k \text{ and }l=j+1\}$ must occur in $\M_G(D_G)$ due to the assumption that $h$ is a homomorphism from $\dom(\M_k(D_k))$ to $\dom(\M_G(D_G))$, we immediately derive that the graph $G$ represented by circuit $C_G$ contains a clique of size $k$, as required. \qed
    \end{proof}

\end{proof}


\section{Conclusions}\label{sec:Conclusion}


We have addressed the problem of UCQ-definability in the OBDM context. To semantically characterize datasets through ontologies even in cases where perfect characterizations do not exist, we have relaxed the notion of perfecteness in terms of recall and precision. Finally, in a scenario that uses the languages commonly adopted in OBDM, we have provided a thorough complexity analysis of three natural, interesting problems associated with the framework.

There are many interesting avenues for future work. 
Some of them are: \myi extending the framework for dealing
also with the \textit{query-by-example} problem, in which two distinct $\lambda^+$ and $\lambda^-$ datasets are given, and one is interested in finding perfect (resp., complete and sound, with their possible corresponding approximations) characterizations queries over the ontology, so that the certain answers of such queries capture all tuples in $\lambda^+$ and no tuple in $\lambda^-$; \myii investigating the existence and the computation problems when we adopt CQ as a query language instead of UCQ; \myiii seeking for techniques that allow to obtain, from end users' perspectives, more intelligible queries as characterizations; and \myiv evaluating the techniques presented in this paper to real world settings.


\section*{Acknowledgements}

This work has been partially supported by the ANR AI Chair INTENDED (ANR-19-CHIA-0014), by MIUR under the PRIN 2017 project ``HOPE'' (prot. 2017MMJJRE), by the EU under the H2020-EU.2.1.1 project TAILOR, grant id.\ 952215, and by European Research Council under the European Union’s Horizon 2020 Programme through the ERC Advanced Grant WhiteMech (No. 834228).

\vfill\eject

\bibliographystyle{abbrv}
\bibliography{bib/string-long,bib/krdb,bib/w3c}

\end{document}